\documentclass[sigconf,screen]{acmart}
\usepackage{microtype}
\usepackage{graphicx}
\usepackage{subfigure}
\usepackage{booktabs} 
\usepackage{multicol,multirow}
\usepackage{bbding}
\usepackage{wrapfig}
\usepackage{amsmath}
\usepackage{mathtools}
\usepackage{amsthm}
\definecolor{dkblue}{rgb}{0,0.08,0.45}

\theoremstyle{plain}
\newtheorem{theorem}{Theorem}[section]

\newtheorem{lemma}[theorem]{Lemma}

\theoremstyle{definition}

\newtheorem{assumption}[theorem]{Condition}
\newtheorem{example}[theorem]{Example}
\theoremstyle{remark}

\DeclareMathOperator*{\argmin}{arg\,min}
\DeclareMathOperator*{\argmax}{arg\,max}
\DeclareMathOperator*{\bias}{\mathrm{Bias}}
\DeclareMathOperator*{\var}{\mathrm{Var}}

\DeclareMathOperator*{\mse}{\text{MSE}}
\DeclareMathOperator*{\pair}{\text{pair}}
\newcommand{\indep}{\perp \!\!\! \perp}
\newcommand{\meanN}{\frac{1}{n} \sum_{i=1}^n}

\newcommand{\sumA}{\sum_{a \in \mathcal{A}}}

\newcommand{\sumM}{\sum_{m \in \mathcal{M}}}
\newcommand{\sumMprime}{\sum_{m' \in \mathcal{M}}}
\newcommand{\sumS}{\sum_{s \in \mathcal{S}}}

\newcommand{\sumL}{\sum_{l = 1}^L}

\newcommand{\sumMBarKth}{\sum_{\phi(\Bar{m}) \in 2^{\calA}}}

\newcommand{\mE}{\mathbb{E}}
\newcommand{\mV}{\mathbb{V}}

\newcommand{\mR}{\mathbb{R}}

\newcommand{\ind}{\mathbb{I}}
\newcommand{\calD}{\mathcal{D}}
\newcommand{\calX}{\mathcal{X}}
\newcommand{\calA}{\mathcal{A}}
\newcommand{\calE}{\mathcal{E}}
\newcommand{\calS}{\mathcal{S}}
\newcommand{\calM}{\mathcal{M}}
\newcommand{\calZ}{\mathcal{Z}}
\newcommand{\calC}{\mathcal{C}}

\newcommand{\fKth}{\hat{f}}

\newcommand{\DeltafKth}{\Delta_{\hat{f}}(x, m, m')}
\newcommand{\DeltaqfKth}{\Delta_{q, \hat{f}}(x, m)}
\newcommand{\DeltaqfprimeKth}{\Delta_{q, \hat{f}}(x, m')}
\newcommand{\piKth}{\pi(\phi(m)|x)}
\newcommand{\pizeroKth}{\pi_0(\phi(m)|x)}
\newcommand{\piBarKth}{\pi(\phi(\Bar{m})|x)}
\newcommand{\piBarzeroKth}{\pi_0(\phi(\Bar{m})\,|\,x)}
\newcommand{\piprimeKth}{\pi(\phi(m')\,|\,x)}
\newcommand{\piprimezeroKth}{\pi_0(\phi(m')|x)}
\newcommand{\piiKth}{\pi(\phi(m_i)|x_i)}
\newcommand{\piizeroKth}{\pi_0(\phi(m_i)|x_i)}

\newcommand{\piizetaphi}{\pi_\zeta(\phi(m_i)|x_i)}
\newcommand{\pizetaphi}{\pi_\zeta(\phi(m)|x)}
\newcommand{\pibarzetaphi}{\pi_\zeta(\phi(\Bar{m})|x)}
\newcommand{\piprimezetaphi}{\pi_\zeta(\phi(m')|x)}

\newcommand{\ips}{\hat{V}_{\mathrm{IPS}} (\pi; \calD)}
\newcommand{\dr}{\hat{V}_{\mathrm{DR}} (\pi; \calD, \hat{q})}
\newcommand{\dm}{\hat{V}_{\mathrm{DM}} (\pi; \calD, \hat{q})}
\newcommand{\PI}{\hat{V}_{\mathrm{PI}} (\pi; \calD)}
\newcommand{\lips}{\hat{V}_{\mathrm{LIPS}} (\pi; \calD)}
\newcommand{\mips}{\hat{V}_{\mathrm{MIPS}} (\pi; \calD)}
\newcommand{\offcem}{\hat{V}_{\mathrm{OffCEM}} (\pi; \calD)}

\newcommand{\lopca}{\hat{V}_{\mathrm{OPCB}}^l (\pi; \calD)}

\newcommand{\Kopca}{\hat{V}_{\mathrm{OPCB}} (\pi; \calD,\phi)}

\newcommand{\Kopcazeta}{\hat{V}_{\mathrm{OPCB}} (\pi_{\zeta}; \calD, \phi)}

\newcommand{\bx}{\boldsymbol{x}}

\newcommand{\bs}{\boldsymbol{s}}

\AtBeginDocument{%
  }

\copyrightyear{2024}
\acmYear{2024}
\setcopyright{acmlicensed}\acmConference[RecSys '24]{18th ACM Conference on Recommender Systems}{October 14--18, 2024}{Bari, Italy}
\acmBooktitle{18th ACM Conference on Recommender Systems (RecSys '24), October 14--18, 2024, Bari, Italy}
\acmDOI{10.1145/3640457.3688099}
\acmISBN{979-8-4007-0505-2/24/10}

\begin{document}

\title{Effective Off-Policy Evaluation and Learning\\in Contextual Combinatorial Bandits}
\renewcommand{\shorttitle}{Effective Off-Policy Evaluation and Learning in Contextual Combinatorial Bandits}


\author{Tatsuhiro Shimizu} 
\authornote{Both authors contributed equally to the paper.}
\affiliation{\institution{Independent Researcher}
\city{Tokyo}
\country{Japan}}
\email{t.shimizu432@akane.waseda.jp}

\author{Koichi Tanaka} 
\authornotemark[1]
\affiliation{\institution{Keio University}
\city{Tokyo}
\country{Japan}}
\email{kouichi_1207@keio.jp}

\author{Ren Kishimoto}
\affiliation{\institution{Tokyo Institute of Technology}
\city{Tokyo}
\country{Japan}}
\email{kishimoto.r.ab@m.titech.ac.jp}

\author{Haruka Kiyohara}
\affiliation{\institution{Cornell University}
\state{NY}
\country{USA}}
\email{hk844@cornell.edu}

\author{Masahiro Nomura}
\affiliation{\institution{CyberAgent, Inc.}
\city{Tokyo}
\country{Japan}}
\email{nomura_masahiro@cyberagent.co.jp}

\author{Yuta Saito}
\affiliation{\institution{Cornell University}
\state{NY}
\country{USA}}
\email{ys552@cornell.edu}

\renewcommand{\shortauthors}{Shimizu et al.}

\begin{abstract}
    We explore off-policy evaluation and learning (OPE/L) in contextual combinatorial bandits (CCB), where a policy selects a subset in the action space. For example, it might choose a set of furniture pieces (a bed and a drawer) from available items (bed, drawer, chair, etc.) for interior design sales. This setting is widespread in fields such as recommender systems and healthcare, yet OPE/L of CCB remains unexplored in the relevant literature. Typical OPE/L methods such as regression and importance sampling can be applied to the CCB problem, however, they face significant challenges due to high bias or variance, exacerbated by the exponential growth in the number of available subsets. To address these challenges, we introduce a concept of factored action space, which allows us to decompose each subset into binary indicators. This formulation allows us to distinguish between the ``main effect'' derived from the main actions, and the ``residual effect'', originating from the supplemental actions, facilitating more effective OPE. Specifically, our estimator, called OPCB, leverages an importance sampling-based approach to unbiasedly estimate the main effect, while employing regression-based approach to deal with the residual effect with low variance. OPCB achieves substantial variance reduction compared to conventional importance sampling methods and bias reduction relative to regression methods under certain conditions, as illustrated in our theoretical analysis. Experiments demonstrate OPCB's superior performance over typical methods in both OPE and OPL.
\end{abstract}



\maketitle

\begin{table*}
  \caption{Comparing CCB with ranking and slate regarding their action spaces and reward observations.}
  \vspace{-3mm}
  \centering
  \scalebox{0.95}{
  \begin{tabular}{c|cccc}
    \toprule
     setting & action space & \# of actions & candidate actions & reward \\
    \midrule
    \textbf{CCB (our focus)}   
    & $\left\{ \{ \phi \}, \{ a \}, \{ b \}, \{ c \}, \{ a, b \}, \{ a, c \}, \{ b, c \}, \{ a, b, c \}  \right\}$ 
    & $2^{|\calA|}$
    & $\calA = \{ a, b, c \}$
    & $r$ \\
    Ranking 
    & $\left\{ (a, b, c), (a, c, b), (b, a, c), (b, c, a), (c, a, b), (c, b, a) \right\}$ 
    & $|\calA|!$ 
    & $\calA = \{ a, b, c \}$
    & $(r_1, \cdots, r_L)$ \\
    Slate  
    & $\begin{Bmatrix}
    (a_1, a_2, a_3), (a_1, a_2, b_3), (a_1, b_2, a_3), (a_1, b_2, b_3), \\
    (b_1, a_2, a_3), (b_1, a_2, b_3), (b_1, b_2, a_3), (b_1, b_2, b_3)
    \end{Bmatrix}$ 
    & $\prod_{l \in [L]} |\calA_l|$ 
    & $\calA_l = \{ a_l, b_l \}$
    & $r$ \\
    \bottomrule
  \end{tabular}
  }\label{tab:combinatorial-ope-difference-from-ranking-slate}
\end{table*}

\section{Introduction}

Personalizing decision-making using past interaction logs is a primary interest in many intelligent systems. In particular, off-policy evaluation and learning (OPE/L), which aim to evaluate or learn a new decision-making policy based only on logged data with no exploration, is considered crucial for mitigating potential risks and ethical concerns encountered in online learning and A/B testing. However, while many methods have been extensively explored in the contextual bandit setting with a single action~\cite{dudik2014doubly, swaminathan2015counterfactual, wang2017optimal, farajtabar2018more, su2019cab, su2020doubly, su2020adaptive, metelli2021subgaussian, saito2022off, saito2023off, saito2021counterfactual, kiyohara2024towards, saito2024long}, more complicated settings involving multiple actions remain sparse in the literature~\citep{swaminathan2017off, mcinerney2020counterfactual, kiyohara2022doubly, kiyohara2023off, kiyohara2024slate}. In particular, one of the most underexplored involves OPE/L for contextual combinatorial bandits (CCB) where a policy chooses a subset in the action space to maximize the reward~\cite{qin2014contextual,chen2013combinatorial}. CCB encompasses the following practical examples.

\begin{example}[Total Outfit Coordination]
    \label{ex-personalized-recommendation}
    When the candidate items are earrings, a pendant, bracelet A, and bracelet B, a policy can recommend, e.g., (1) only earrings, (2) earrings and pendant, (3) both bracelets A and B, or (4) all items, depending on the user's profile to maximize the revenue or profit as the reward.
\end{example}

\begin{example}[Precision Medicine]
    \label{ex-precision-medicine}
    Given medical data for each patient, medical institutions prescribe a combination of medications, e.g., (1) a fever reducer and cough suppressant, (2) cough suppressant, sinus medicine, and flu medicine, or (3) no medicine, to facilitate fast recovery from the disease.
\end{example}

OPE/L using logged data in these CCB problems presents significant challenges due to the complexities of managing the action subset space, which is likely to be vast. Two typical OPE approaches, the Direct Method (DM)~\citep{beygelzimer2009offset} and Inverse Propensity Scoring (IPS)~\citep{horvitz1952generalization}, aim to estimate the policy value by imputing counterfactual rewards via regression or applying importance sampling in the action subset space, respectively. However, DM often encounters challenges with inaccurate regression due to model misspecification and the sparsity of the rewards. In contrast, IPS faces a critical variance issue, with importance weights growing exponentially as the number of candidate subsets increases. Doubly Robust (DR)~\citep{dudik2011doubly} represents another approach, utilizing the reward regressor as a control variate to mitigate the variance issue inherent in IPS. Despite this, DR can still exhibit high variance similar to IPS due to its importance weighting term, and thus all conventional baselines struggle with the large action space either by their large bias or variance~\citep{saito2022off, saito2023off, saito2024hyperparameter, saito2024potec}, which is typical in our CCB setup.\\

\textbf{Contributions.} \quad To address the aforementioned challenges in handling the action subset space, we first formulate the problem of CCB via factorizing a subset of actions into the product of binary indicators that indicate if each action is included in the subset. This formulation enables us to focus on some of the influential actions referred to as the \textit{main actions}, such as cold medicine and fever reducer for flu patients, among other supplemental medicines. We then introduce the concept of decomposing the expected reward of a subset of actions into the \textit{main} and \textit{residual} effects, which are the outcomes of the main actions and that of supplemental actions. Leveraging this reward function decomposition, we develop a novel estimator for OPE of CCB called \textit{\textbf{O}ff-\textbf{P}olicy estimator for \textbf{C}ombinatorial \textbf{B}andits (\textbf{OPCB})}, which applies importance sampling only to the main actions to reduce the exponential variance of IPS, while dealing with the residual effect via regression to reduce the bias. OPCB ensures unbiasedness under a condition called \textit{conditional pairwise correctness}, which requires that the regression model accurately estimate the relative reward difference of the subset of actions when the main actions are identical. In addition to the development of OPCB, considering a realistic situation where there is no prior knowledge about the main actions, we discuss a data-driven procedure to identify an appropriate set of main actions to minimize the mean squared error (MSE) of OPCB. Finally, empirical results on synthetic and real-world data reflecting the CCB problem demonstrate that our method evaluates a new CCB policy more precisely and learns a new CCB policy more efficiently than existing methods for a range of experiment configurations.\\

\textbf{Related work.}
Combinatorial bandits to choose a subset in the action space to maximize the reward has been studied much in the online learning literature~\cite{chen2013combinatorial,qin2014contextual}, but there is no existing work discussing OPE/L in this setup thus far.
Ranking OPE~\citep{li2018offline, mcinerney2020counterfactual, kiyohara2022doubly, kiyohara2023off} and slate OPE~\citep{swaminathan2015batch, vlassis2021control, kiyohara2024slate} study similar types of OPE/L problems. However, as shown in Table~\ref{tab:combinatorial-ope-difference-from-ranking-slate}, the problem of CCB differs from both regarding the form of the action space.
First, the ranking setup considers evaluating a ranking policy that aims to optimize the ranking of a given action space $\calA$, and in the logged data, the rewards corresponding to each position in the ranking are observable, which is unavailable in CCB. Thus, estimators developed for the ranking setup~\citep{li2018offline, mcinerney2020counterfactual, kiyohara2022doubly, kiyohara2023off} cannot be applied to solve OPE/L in CCB.
Second, the slate setup considers choosing $L$ different actions, each from the corresponding action space $\calA_l$, while in our setting of CCB, there is only one action space $\calA$. We discuss an interesting connection between our formulation and that of slate OPE in detail in the appendix. 

We also differentiate our contributions from those of \citet{saito2023off}, which develops the OffCEM estimator to deal with OPE in large action spaces in a single action setup. Even though our main idea of reward function decomposition is inspired by OffCEM, its application to the CCB setup is non-trivial and cannot be done without our formulation with action space factorization. Moreover, we discuss how to identify an appropriate reward decomposition in a data-driven fashion and perform associated empirical evaluations, which are missing in \citet{saito2023off}. We also propose an extension of our estimator to an OPL method while \citet{saito2023off} focused solely on the OPE problem. Therefore, our work is the first to formulate the OPE problem in the CCB setup and offer multiple unique contributions in both methodological and empirical perspectives.

\section{Problem Formulation}
This section formulates the problem of OPE in the contextual combinatorial bandits (CCB).
Let $x \in \calX \subseteq \mathbb{R}^{d_x}$ be a context vector such as user demographics, sampled from an unknown distribution $p(x)$. We use $\calA$ to denote some given set of actions, and $s \in \calS$ to denote a subset in $\calA$. For example, when $\calA = \{ a_1, a_2 \}$, the corresponding subset space is $\calS = \{ \{ \emptyset \}, \{ a_1 \}, \{ a_2 \}, \{ a_1, a_2 \}  \}$. Given a context $x$, the decision-maker, such as a recommendation interface, chooses a subset in the action space, $s$, following a stochastic policy $\pi$, where $\pi(s \,|\, x)$ is the probability of choosing a subset $s$ in $\calA$ given context $x$. Let $r \in [r_{\text{min}}, r_{\text{max}}]$ be a reward, drawn from an unknown conditional distribution $p(r\,|\,x,s)$. We typically measure the effectiveness of policy $\pi$ by the following \textit{policy value}:
\begin{align*}
    V(\pi) \coloneq \mE_{p(x) \pi(s|x) p(r|x, s)}\left[ r \right] = \mE_{p(x) \pi(s|x)}\left[ q(x, s) \right], 
\end{align*}
where $q(x, s) := \mE \left[ r \,|\, x, s \right]$ is the expected reward given $x$ and $s$.

The goal of OPE/L is to accurately estimate $V(\pi)$ of target policy $\pi$ or to maximize $V(\pi)$ using only offline logged data. The logged data $\calD = \{(x_i, s_i, r_i)\}_{i = 1}^n \sim \prod_{i=1}^n p(x_i) \pi_0(s_i|x_i) p(r_i|x_i, s_i)$ contains $n$ i.i.d. observations collected under a \textit{logging policy} $\pi_0$, which is often different from $\pi$. 
Particularly in OPE, we aim at developing an estimator $\hat{V}(\pi; \calD)$ that minimizes the MSE as an accuracy measure:
\begin{align*}
    \mse \left( \hat{V}(\pi; \calD) \right) 
    &:= \mE_{\calD} \left[ \left( \hat{V}(\pi; \calD) - V(\pi) \right)^2 \right] \\
    &= \bias \left[ \hat{V}(\pi; \calD) \right]^2 + \var \left[ \hat{V}(\pi; \calD) \right], 
\end{align*}
where $\mE_{\calD}[\cdot]$ takes the expectation over the distribution of $\calD$. 
The bias and variance terms are defined as follows.
\begin{align*}
    \bias \left[ \hat{V}(\pi; \calD) \right] &:= \mE_{\calD} \left[ \hat{V}(\pi; \calD) \right] - V(\pi) \\
    \var \left[ \hat{V}(\pi; \calD) \right] &:= \mE_{\calD} \left[ \left( \hat{V}(\pi; \calD) - \mE_{\calD} \left[ \hat{V}(\pi; \calD) \right] \right)^2 \right].
\end{align*}
In the following, we first focus on developing a novel estimator to solve the OPE problem in CCB and later show that the proposed estimator can readily be extended to solve the OPL problem.

\subsection{Applications of Typical Ideas}
\label{sec:baselines}

Although there is no existing literature that formally studies OPE for CCB, we discuss how to apply the conventional approaches, DM~\cite{beygelzimer2009offset}, IPS~\cite{horvitz1952generalization}, and DR~\cite{dudik2014doubly} to this setup, and their limitations.

DM employs a regression model $\hat{q}(x, s) (\approx q(x, s))$ to estimate the policy value by imputing the missing rewards as follows:
\begin{align*}
    \dm := \meanN \mE_{\pi(s|x_i)} [\hat{q}(x_i, s)] 
    = \meanN \sumS \pi(s|x_i) \hat{q}(x_i, s) .
\end{align*}
where $\hat{q}$ is optimized to minimize the estimation error against the rewards $r_i$ observed in the logged data.
DM has low variance compared to other typical approaches, and is accurate when $\hat{q}$ is accurate for all possible $(x, s) \in \calX \times \calS$. However, DM incurs considerable bias because it becomes extremely difficult to accurately estimate the expected reward of subset $s'$ that is not observed in the logged data, i.e., $s' \neq s_i$. This issue becomes particularly problematic as the number of subsets grows exponentially~\citep{strehl2010learning,saito2021evaluating}.

In contrast, IPS re-weighs the rewards observed in the logged data, $r_i$, using the ratio of the probabilities of observing subset $s_i$ under the logging policy $\pi_0$ and target policy $\pi$:
\begin{align*}
    \ips := \meanN \frac{\pi(s_i\,|\,x_i)}{\pi_0(s_i\,|\,x_i)} r_i = \meanN w(x_i, s_i) r_i, 
\end{align*}
where $w(x, s) = \pi(s\,|\,x) / \pi_0(s\,|\,x)$ is called the vanilla importance weight. IPS is unbiased under the common support condition: $\pi(s\,|\,x) > 0 \implies \pi_0(s\,|\,x) > 0,\; \forall (x, s) \in \mathcal{X} \times \mathcal{S}$. However, in order to satisfy this condition in the entire subset space $\calS$, $\pi_0$ has to allocate extremely small values to some of the subsets, resulting in substantial weight values and variance issue~\citep{saito2022off, saito2023off}.

DR is the third baseline that is considered better than DM and IPS, and is defined as follows.
\begin{align*}
    \dr := \meanN \left\{ w(x_i, s_i) (r_i - \hat{q}(x_i, s_i)) + \mE_{\pi(s|x_i)}[\hat{q}(x_i, s)] \right\}.
\end{align*}
DR is unbiased under the same support condition as IPS. Moreover, by employing the reward model $\hat{q}$, DR often reduces variance compared to IPS under a mild condition regarding the regression accuracy. However, the variance of DR remains extremely high when the action space is large~\citep{saito2022off, saito2023off} as in our problem of CCB.

\section{Our Approach}
In the previous section, we have seen the critical challenges of OPE in the problem of CCB. To tackle them, we first introduce a formulation of CCB based on the factored action space and binary indicators that indicate if each action in $\calA$ is included in the selected subset $s$. 
Specifically, we factorize the action subset space $\calS$ as

\begin{align*}
    \mathcal{S} = \underbrace{\{ \emptyset, a_1 \}}_{\mathcal{M}_1} \times \underbrace{\{ \emptyset, a_2 \}}_{\mathcal{M}_2} \times \cdots \times \underbrace{\{ \emptyset, a_L \}}_{\mathcal{M}_L}
    = \prod_{l=1}^L \mathcal{M}_l := \calM
\end{align*}
where $m_l \in \calM_l = \{ \emptyset, a_l \} $ is a binary indicator; $m_l = \emptyset$ means that the action $a_l$ is not chosen in the subset ($a_l \notin s$), while $m_l = a_l$ means the opposite ($a_l \in s$). The advantage of this factorized formulation is that it becomes possible to focus more on the indicators of specific key actions and to be able to treat them differently from other supplemental actions when constructing an estimator.

\begin{figure}[t]
    \includegraphics[scale=0.1]{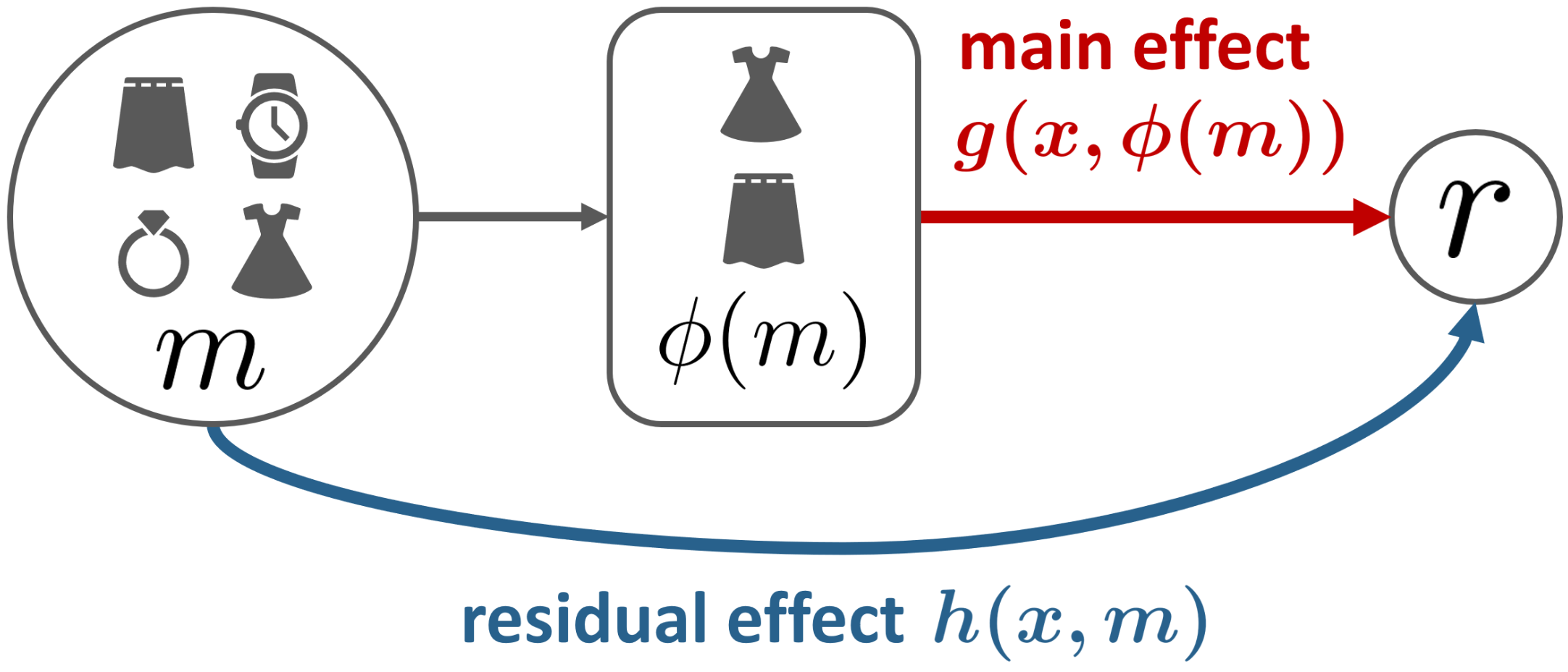}
    \centering
    \Description{}
    \caption{The graphical model of reward function decomposition into the main and residual effects, where $\phi(m)$ is the set of main actions. In this example, we consider a total outfit coordination where we recommend a combination of the fashion items, and we regard the tops and bottoms as the main actions, whereas the residual effect captures the effect of the accessories on the reward.}
    \centering
    \label{fig:decomposition}
\end{figure}

\subsection{The OPCB Estimator}

The primary idea of our OPCB estimator is to extract the main elements from the action set and distinguish them from other supplemental elements. For instance, a patient's recovery from a disease can mostly depend on some important medicines, but is also affected by some adjunctive medicines. To distinguish these effects, we consider the following decomposition of the expected reward.
\begin{align}
    q(x, m) = \underbrace{g(x, \phi(m))}_{\text{main effect}} \;+\; \underbrace{h(x, m)}_{\text{residual effect}} \label{eq:decomposition}, 
\end{align}
where $g(x, \phi(m))$ is the main effect accounted by the set of main actions $\phi(m)$, and $h(x, m)$ is the residual effect that is not captured only by the main actions, as illustrated in Figure \ref{fig:decomposition}. Note that we do not impose any restrictions of the functional form of $g$ and $h$; Eq.~\eqref{eq:decomposition} is not an assumption.
Note also that the function to extract the main actions, $\phi(m)$, is currently assumed somewhat given, but we will discuss how to optimize it from the logged data after analyzing the proposed estimator.

Based on this reward function decomposition, OPCB estimates the main and residual effects via importance sampling and regression, respectively, as follows.
\begin{align}
    \Kopca := \meanN \Bigg\{ w(x_i,\phi(m_i)) \left( r_i - \hat{f}(x_i, m_i) \right) \notag\\
    \hspace{4.5cm} + \mE_{\pi(m|x_i)} \left[ \hat{f}(x_i,m) \right] \Bigg\}, \label{eq:K-th-order-OPCB}
\end{align}
where $\hat{f}(x, m)$ is some regression model whose optimization is discussed based on our theoretical analysis given in the following section. $w(x, \phi(m)) := \pi(\phi(m)\,|\,x) / \pi_0(\phi(m)\,|\,x)$ is the marginalized importance weight regarding the main actions $\phi(m)$, where $\pi(\phi(m)\,|\,x) := \sum_{m' \in \calM} \pi(m'\,|\,x) \cdot \ind \{ \phi(m)  =  \phi(m')  \}$. 
The rationale behind using two different approaches, including importance sampling and regression, to estimate the main and residual effects is to (1) unbiasedly estimate the main effect while achieving substantial variance reduction from IPS and DR, and (2) to deal with the residual effect via regression without introducing much variance. The following section formally discusses the theoretical properties of OPCB and advantages of reward decomposition in CCB.

\subsection{Theoretical Analysis}\label{sec:theoretical-analysis}

We first show OPCB's unbiasedness under the following conditions.

\begin{assumption}[Common Support w.r.t the Main Actions]
    \label{ass.common_support_Kth}
    The logging policy $\pi_0$ is said to satisfy common support w.r.t the main actions for policy $\pi$ if $\pi(\phi(m)\,|\,x) > 0 \implies \pi_0(\phi(m)\,|\,x)> 0$ for all $m \in \calM$ and $x \in \calX$. 
\end{assumption}

\begin{assumption}[Conditional Pairwise Correctness]
    \label{ass.local_correctness_Kth}
    A regression model $\hat{f}: \calX \times \calM \to \mR$ is said to have conditional pairwise correctness if the following holds for all $x \in \calX$ and $m, m' \in \calM$:
    \begin{align}
        \phi(m) = \phi(m') \quad \Rightarrow \quad \Delta_q(x, m, m') = \DeltafKth, 
    \end{align}
    where $\Delta_q(x, m, m') := q(x, m) - q(x, m')$ is the relative difference of the expected rewards between $m$ and $m'$, and $\DeltafKth := \fKth(x, m) - \fKth(x, m')$ is its estimate.
\end{assumption}

\begin{theorem}[Unbiasedness of the OPCB estimator]
    \label{thm.unbiasedness_OPCB_Kth}
    Under Conditions \ref{ass.common_support_Kth} and \ref{ass.local_correctness_Kth}, the OPCB estimator ensures unbiasedness. i.e., $\mE_{\calD}[\Kopca] = V(\pi)$. See the appendix for the proof.
\end{theorem}

The common support w.r.t the main actions (Condition~\ref{ass.common_support_Kth}) ensures unbiased estimation of the main effect $g(x,\phi(m))$. Since this requires common support regarding only the main actions, this is milder than the common support condition of IPS. Conditional pairwise correctness (Condition~\ref{ass.local_correctness_Kth}) ensures unbiased estimation of the residual effect $h(x,m)$. This condition is milder than the condition for DM to become unbiased because DM requires $\hat{q}(x,s)=q(x,s)$ for all $(x, s)$. Moreover, even when conditional pairwise correctness is not satisfied, we can characterize the bias of OPCB as follows.

\begin{theorem}[Bias of OPCB]    
    \label{thm.bias_OPCB_Kth}
    Under Condition \ref{ass.common_support_Kth}, the OPCB estimator has the following bias.
    \begin{align}
        &\bias\left( \Kopca \right) = \mE_{p(x) \pi(\phi(\Bar{m})|x)} \Bigg[ \sum_{\substack{m < m':\\ \phi(m) = \phi(m') = \phi(\Bar{m})}} \notag  \\
        &\underbrace{\pi_0(m|x, \phi(\Bar{m})) \pi_0(m'|x, \phi(\Bar{m}))}_{(i) \text{ stochasticity of the logging policy given $\phi(\Bar{m})$}} \notag \\
        \times & \underbrace{\Bigg( \frac{\pi(m'|x, \phi(\Bar{m}))}{\pi_0(m'|x, \phi(\Bar{m}))} - \frac{\pi(m|x, \phi(\Bar{m}))}{\pi_0(m|x, \phi(\Bar{m}))} \Bigg)}_{(ii) \text{ difference in the distribution shift between $\pi$ and $\pi_0$ given $\phi(\Bar{m})$}} \notag\\
        \times & \underbrace{\left( \Delta_q(x, m, m') - \DeltafKth \right)}_{(iii) \text{  estimation error of $\hat{f}(x,m)$ w.r.t. relative reward difference}}  \Bigg] , \label{eq:bias-opca}
    \end{align}
    where $\pi(m\,|\,x, \phi(\Bar{m})) = \pi(m\,|\,x) \ind \{\phi(m) = \phi(\Bar{m})\} / \pi(\phi(\Bar{m})\,|\,x)$ and $m, m' \in \calM$. See the appendix for the proof.
\end{theorem}
Theorem \ref{thm.bias_OPCB_Kth} implies that three factors contribute to the bias of OPCB. The first factor $(i)$ is the stochasticity of the logging policy conditional on the main actions ($\phi(\Bar{m})$), which implies that the bias of OPCB becomes minimal when the conditional policy, $\pi_0(m\,|\,x, \phi(\Bar{m}))$, is near deterministic. The second factor $(ii)$ is the difference in the marginal importance weights of two action subsets that share the same main actions. This indicates that the bias becomes small when the distribution shift between policies regarding the supplemental actions is small. The final factor $(iii)$ is estimation accuracy of the regression model $\hat{f}(x, m)$ against the relative difference in expected rewards $\Delta_q(x, m, m')$ only for $\phi(m) = \phi(m')$, suggesting that a more accurate estimation of the residual effects by $\hat{f}(x, m)$ results in a smaller bias of OPCB.

In contrast, the variance of OPCB is characterized by the scale of the marginalized importance weight and another aspect of $\hat{f}(x, m)$.

\begin{theorem}[Variance of OPCB]
    \label{thm.variance_OPCB_Kth}
    Under Conditions \ref{ass.common_support_Kth} and \ref{ass.local_correctness_Kth}, the variance of the OPCB estimator is as follows.
    \begin{align}
        n \mV_{\calD} \left[ \Kopca \right] &= \mE_{p(x) \pi_0(m|x)} [w(x, \phi(m))^2 \sigma^2(x, m)] \notag\\
        +& \mE_{p(x)} \left[ \mV_{\pi_0(m|x)} \left[w(x, \phi(m)) \DeltaqfKth \right] \right] \notag\\
        +& \mV_{p(x)} \left[ \mE_{\pi(m|x)} \left[q(x, m)\right] \right] \label{eq:variance-OPCB-Kth}, 
    \end{align}
    where $\DeltaqfKth := q(x, m) - \fKth(x, m)$ is the estimation error of $\fKth(x, m)$ against the expected reward function $q(x, m)$ and $\sigma^2(x, m) := \mV_{p(r|x, a)}[r]$. See the appendix for the proof.
\end{theorem}

Theorem \ref{thm.variance_OPCB_Kth} demonstrates that the variance of OPCB depends only on the importance weights of the main actions, i.e., $w(x, \phi(m))$. Since OPCB does not consider the importance weight regarding supplemental actions, we can expect a significant variance reduction compared to IPS and DR, which rely on the importance weight regarding the entire action subsets. Moreover, Eq. \eqref{eq:variance-OPCB-Kth} shows that the variance of OPCB depends on the accuracy of the regression model $\hat{f}(x, m)$ against the expected reward $q(x,m)$, not against the relative reward difference as in the bias expression. This implies that we can minimize the second term of the variance by minimizing $\DeltaqfKth$ when deriving the regression model $\hat{f}(x, m)$.

\subsection{Optimizing Regression and Decomposition}
Based on the theoretical analysis, this section discusses how to optimize a regression model $\hat{f}(x, m)$ to minimize the bias and variance of OPCB. We then discuss how to identify the appropriate main actions $\phi(m)$ to decompose the expected reward based only on the logged data to minimize the MSE of OPCB.

\subsubsection{Optimizing the Regression Model} \label{subsec:two-step regression}
The theoretical analysis suggests that optimizing the accuracy of the regression model $\hat{f}(x, m)$ against the relative reward difference $\Delta_q(x, m, m')$ and expected reward $q(x,m)$ results in the reduction of the bias and variance of OPCB, respectively. We first deal with the bias of OPCB and minimize it via performing the following pairwise regression problem towards minimizing the estimation error against $\Delta_q(x, m, m')$:
\begin{align*}
    \psi^* = \argmin_{\psi \in \Psi} \sum_{(x, m, m', r_m, r_{m'}) \in \calD_{\pair}} \ell_{h} \Big( r_m - r_{m'}, \hat{h}_{\psi}(x, m) - \hat{h}_{\psi}(x, m') \Big), 
\end{align*}
where $\hat{h}_{\psi}: \calX \times \calM \to \mR$ is parametrized by $\psi$ and $\ell_{h}$ is some appropriate loss function such as the L2 loss. The pairwise dataset $\calD_{\text{pair}}$ can be obtained from the original logged data $\calD$ as follows.
\begin{align*}
    \calD_{\pair} := \Bigg\{ (x, m, m', r_m, r_{m'}) \Bigg| \begin{array}{l} 
      (x_m, m, r_m), (x_{m'}, m', r_{m'}) \in \calD \\
    x = x_m = x_{m'}, \phi(m) = \phi(m')
    \end{array} 
    \Bigg\}.
\end{align*}

Once we have done the pairwise regression procedure to minimize the bias of OPCB, we go on to deal with its variance minimization via minimizing the error of the regression model $\hat{f}(x, m)$ against the absolute expected reward $q(x,m)$ reflecting our variance analysis. Specifically, this can be done via the following ordinary regression against the rewards observed in the logged data:
\begin{align*}
    \theta^{\ast} = \argmin_{\theta \in \Theta} \sum_{(x, m, r) \in \calD} \ell_{g} \left( r - \hat{h}_{\psi^{\ast}}(x, m), \hat{g}_{\theta}(x, \phi(m)) \right).
\end{align*}
After performing these regression procedures, we can define the regression model as $\hat{f}_{\theta^{\ast}, \psi^{\ast}}(x,m) = \hat{g}_{\theta^{\ast}}(x,\phi(m)) + \hat{h}_{\psi^{\ast}}(x,m)$, which minimizes the error against the relative reward difference and absolute expected reward to optimize both the bias and variance of OPCB leveraging its theoretical results.

\subsubsection{Optimizing the Set of Main Actions}
\label{subsec:hyperparameter-tuning}

The previous sections defined OPCB assuming that the set of main actions $\phi(m)$ is appropriately given. However, in practical applications, identifying the appropriate main actions to perform OPCB is non-trivial, even though this identification is crucial as we will see in the empirical section. 
When performing such an optimization, we should ideally optimize $\phi$ to minimize the MSE of OPCB, recalling that our technical goal is to build an estimator that minimizes the MSE:
\begin{align}
    \!\!\phi^* = \argmin_{\phi \in \Phi} \bias \left( \Kopca \right)^2 + \var \left( \Kopca \right). \label{eq:ideal-tuning}
\end{align}
However, the true MSE of any estimator is unknown as it depends on $V(\pi)$, so we propose to optimize the set of main actions ($\phi$) via the following empirical formula.
\begin{align}
    \! \hat{\phi} = \argmin_{\phi \in \Phi} \widehat{\bias} \left( \Kopca \right)^2 + \widehat{\var} \left( \Kopca \right), \label{eq:tuning}
\end{align}
where $\Phi$ is a (finite) set of candidate functions $\phi$ and $\widehat{\var}(\cdot)$ is the sample variance of OPCB, which is defined as~\citep{wang2017optimal}:
\begin{align*}
    \widehat{\var} \left( \Kopca \right) := \frac{1}{n^2} \sum_{i=1}^n \left(Y_{i} - \Kopca \right)^{2}.
\end{align*}
where $Y_{i}=w(x_i,\phi(m_i)) ( r_i - \hat{f}(x_i, m_i) ) + \mE_{\pi(m|x_i)} [\hat{f}(x_i, m)]$. In contrast, $\widehat{\bias}(\cdot)$ is an estimate of the bias of OPCB. There exist several methods to estimate the bias of an estimator using only the logged data~\citep{udagawa2023policy,su2020adaptive,cief2024cross,felicioni2024autoope}. In the experiments, we empirically demonstrate that OPCB substantially outperforms the existing methods in a variety of environments even with a noisy estimate of the bias when performing Eq.~\eqref{eq:tuning}.

\subsection{Extension to Off-Policy Learning in CCB}\label{sec:opl}
We have thus far focused on evaluating the policy value of a new CCB policy using only the logged data by OPCB. In this section, we extend it to solve the OPL problem to learn a new CCB policy, aimed at optimizing the policy value. In particular, we use a policy gradient approach to learn the parameter $\zeta$ of a parametrized CCB policy $\pi_{\zeta}(m\,|\,x)$ to maximize the policy value, i.e., $$\zeta^* = \argmax_{\zeta}\, V (\pi_{\zeta}).$$ The policy gradient approach updates the policy parameter via iterative gradient ascent as $\zeta_{t + 1} \gets \zeta_{t} + \eta \nabla_{\zeta} V(\pi_{\zeta_{t}})$, where $\eta$ is the learning rate. Since we have no access to the true policy gradient $\nabla_{\zeta} V(\pi_{\zeta}) = \mE_{p(x)\pi_{\zeta}(m|x)} \left[ q(x,m) \nabla_{\zeta} \log \pi_{\zeta}(m\,|\,x) \right]$, we need to estimate it using $\calD$. We achieve this by extending OPCB as
\begin{align}
    &\nabla_\zeta \hat{V}_{\mathrm{OPCB}} (\pi_\zeta;\calD,\phi) \label{eq:opcb-pg} \\
    &:= \meanN \bigg\{ \frac{\pi_{\zeta}(\phi(m_i)\,|\,x_i)}{\pi_0(\phi(m_i)\,|\,x_i)} \left( r_i - \hat{f}(x_i,m_i) \right) \log \pi_{\zeta}(\phi(m_i)|x_i) \notag \\
    & \hspace{3cm} + \mE_{\pi_\zeta(m|x_i)} \left[ \hat{f}(x_i,m) \nabla_{\zeta} \log \pi_{\zeta}(m\,|\,x_i)  \right] \bigg\}\notag
\end{align}
 We can derive the bias and variance of this policy-gradient estimator similarly to the analysis of OPCB, which can be found in the appendix. Note that we call our policy gradient (PG) method to solve OPL based on Eq.~\eqref{eq:opcb-pg} as OPCB-PG.

\section{Synthetic Experiments}

This section evaluates OPCB using synthetic data and identifies the situations where OPCB becomes particularly more effective in OPE/L compared to the baseline methods. Note that in our experiments, we focused on settings with a relatively small number of ``unique'' actions ($|\calA| \le 10$). This might seem small compared to real-world applications, but existing methods cannot even handle this due to their variance. We believe OPE/L for CCB with large unique action spaces, potentially leveraging structure in $\calA$ as studied by~\citep{saito2022off,saito2023off,saito2024potec,sachdeva2024off,kiyohara2024slate}, would be an interesting future topic.

\subsection{Synthetic Data Generation \label{subsec:synthetic ope}}
To generate synthetic data, we first define 200 synthetic users characterized by 5-dimensional context vectors ($x$) sampled from the standard normal distribution. Then, for each action subset $m = (m_1, \cdots, m_L) \in \calM$ (or equivalently $s \in \calS$), we simulate its expected reward via synthesizing the main and residual effects:
\begin{align}
    \label{expected reward function}
    q(x, m) = \lambda \cdot g(x, \phi_{true}(m)) + (1 - \lambda) \cdot h(x, m), 
\end{align}
where $g(x,\phi_{true}(m))$ and $h(x, m)$ in Eq.~\eqref{expected reward function} are defined rigorously in the appendix. $\lambda$ is an experiment parameter to control the contribution of the main effect, where the number of main actions in the true reward function is $|\phi_{true}(m)|=3$. 
Based on the above synthetic reward function, we sample reward $r$ from a normal distribution whose mean is $q(x,m)$ and standard deviation $\sigma$ is 3.0. \\
The logging and target policies are defined respectively as follows.
\begin{align}
    \pi_0(m\,|\,x) 
    &= \frac{\exp(\beta \cdot q(x,m))}{\sum_{m' \in \calM} \exp(\beta \cdot q(x,m'))}, \label{eq:logging} \\
    \pi(m\,|\,x) 
    &= (1-\epsilon) \cdot \ind \{m = \underset{m' \in \calM} {\operatorname{argmax}} \ q(x,m')\} + \frac{\epsilon}{|\calM|} \label{eq:evaluation}.
\end{align}
$\beta$ and $\epsilon$ are the experiment parameters to control the stochasticity and optimality of these policies. We use $\beta = -0.5$ and $\epsilon = 0.2$, which mean that the logging policy performs worse than the uniform random policy, while the target policy is close to optimal.

\begin{figure*}[t]
    \includegraphics[scale=0.25]{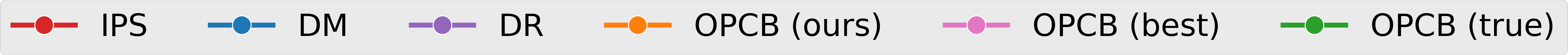} 
    \Description{}
    \vspace{3mm}
    \includegraphics[width=0.85\linewidth]{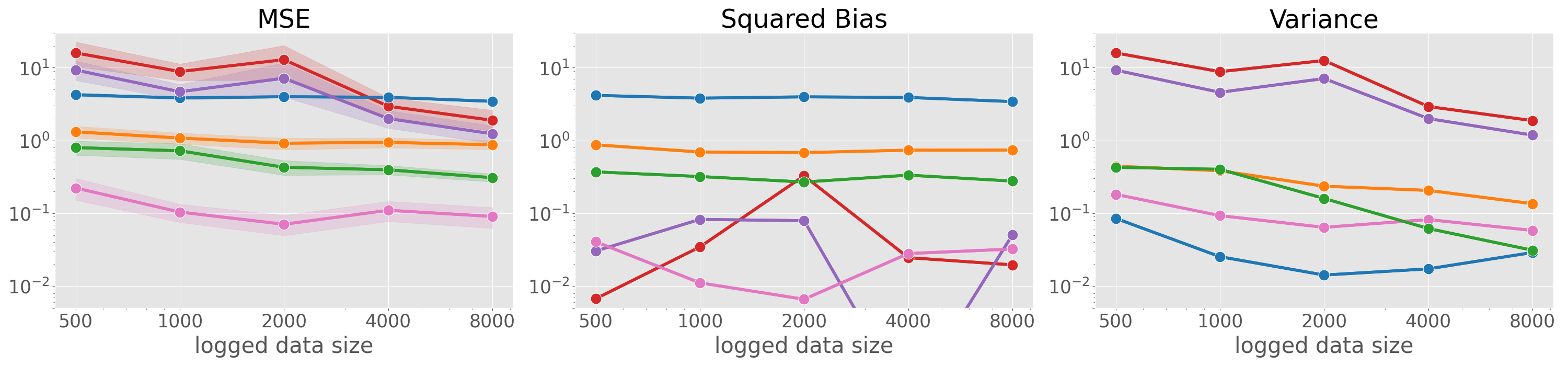}
    \centering \vspace{-3mm}
    \caption{Comparison of the estimators' MSE, Squared Bias, and Variance with varying logged data sizes ($n$).} \vspace{3mm}
    \centering
    \label{fig:n_samples}
    \includegraphics[width=0.85\linewidth]{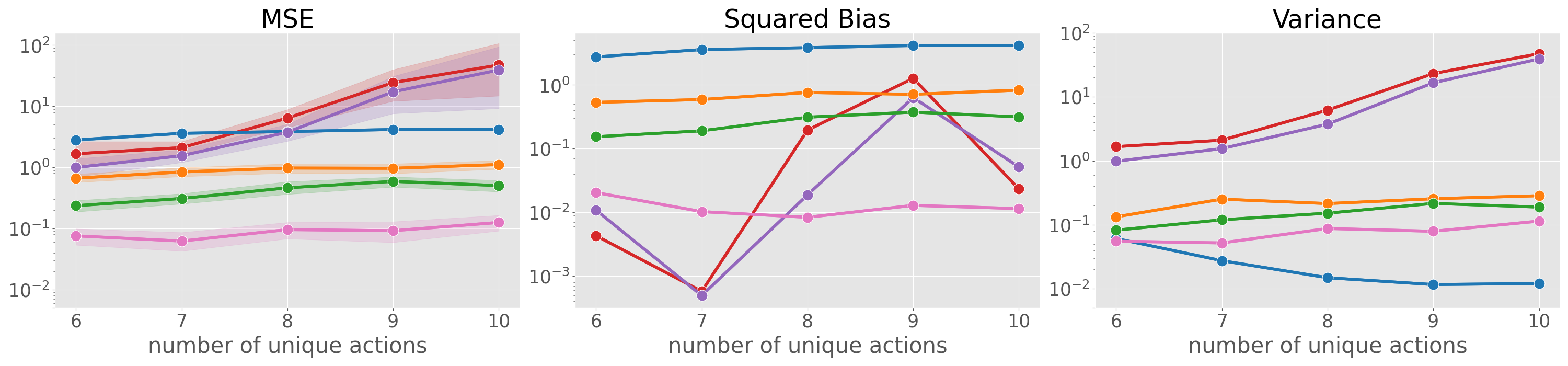}
    \centering \vspace{-3mm}
    \caption{Comparison of the estimators' MSE, Squared Bias, and Variance with varying numbers of actions ($|\calA|$).\vspace{3mm}} 
    \centering
    \label{fig:n_actions}
    \includegraphics[width=0.85\linewidth]{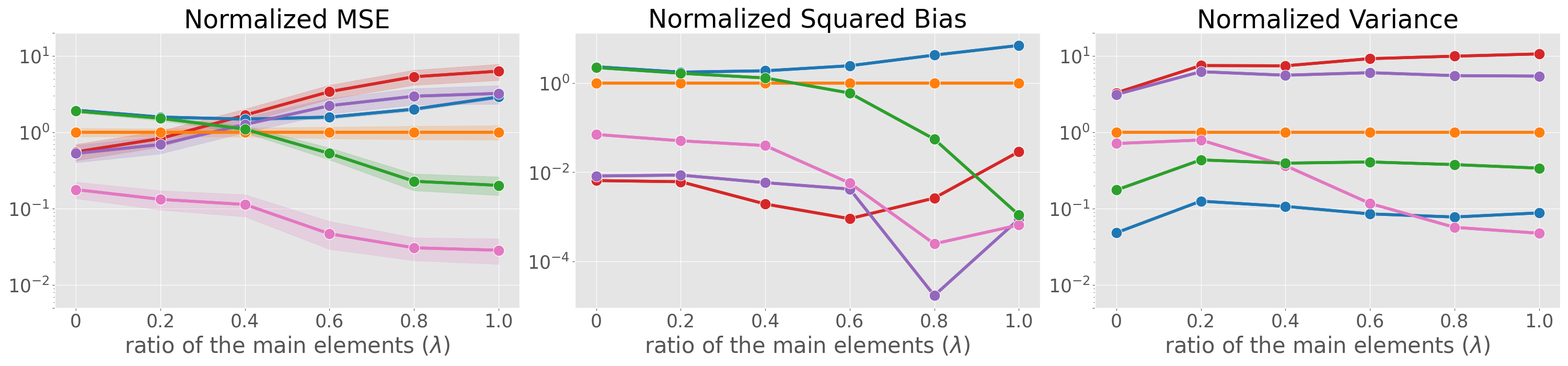}
    \centering \vspace{-3mm}
    \caption{Comparison of the estimators' MSE, Squared Bias, and Variance (normalized by those of OPCB (ours)) with varying values of $\lambda$ in Eq. \eqref{expected reward function}.}
    \centering
    \label{fig:lambda}
\end{figure*}

\paragraph{Compared methods.}
We compare OPCB with DM, IPS, and DR where all of these baseline methods are defined rigorously in Section~\ref{sec:baselines}. The regression models used in DM, DR, and OPCB are parameterized by a 3-layer neural network. Note that we optimize the regression model for OPCB via the two-stage procedure described in Section~\ref{subsec:two-step regression}. Regarding OPCB, we report the results of its three variants, "OPCB (true)", "OPCB (best)", and "OPCB (ours)", each based on different methods to define the main actions $\phi(m)$. 
Specifically, OPCB (true) uses $\phi_{true}$, which defines the true synthetic reward function as in Eq.(\ref{expected reward function}), meaning that it leverages the accurate knowledge about the reward function, which is usually unavailable. OPCB (best) derives the main actions using the true MSE of OPCB, i.e., $\phi^*$ as in Eq.~\eqref{eq:ideal-tuning}. This method is also infeasible in practice because we do not have access to the true MSE. Therefore, OPCB (best) provides the best achievable accuracy based on our framework as a reference. Finally, OPCB (ours) optimizes the main actions based on an estimated MSE to investigate the effectiveness of OPCB with a data-driven method as described in Section~\ref{subsec:hyperparameter-tuning}, which is feasible in practice. Specifically, OPCB (ours) simulates the bias estimation when performing Eq.~\eqref{eq:tuning} as follows.
\begin{align}
    \widehat{\bias} \left(\Kopca\right)^2 = \bias \left(\Kopca\right)^2 + \delta_{\phi} \label{eq:bias_error}
\end{align}
where $\delta_{\phi}$ is a noise sampled from a normal distribution with mean $0$ and standard deviation $\sigma$. We use random search and try 30 candidate functions $\phi$ when performing OPCB (ours).

\begin{figure*}[t]
    \includegraphics[scale=0.23]{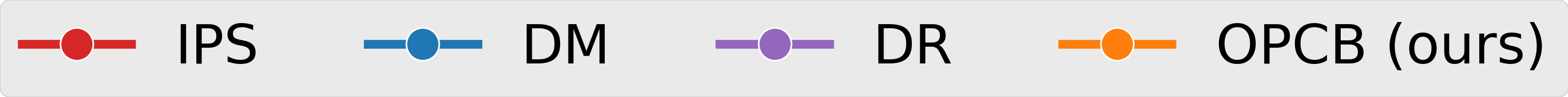} \vspace{3mm}
    \includegraphics[width=0.85\linewidth]{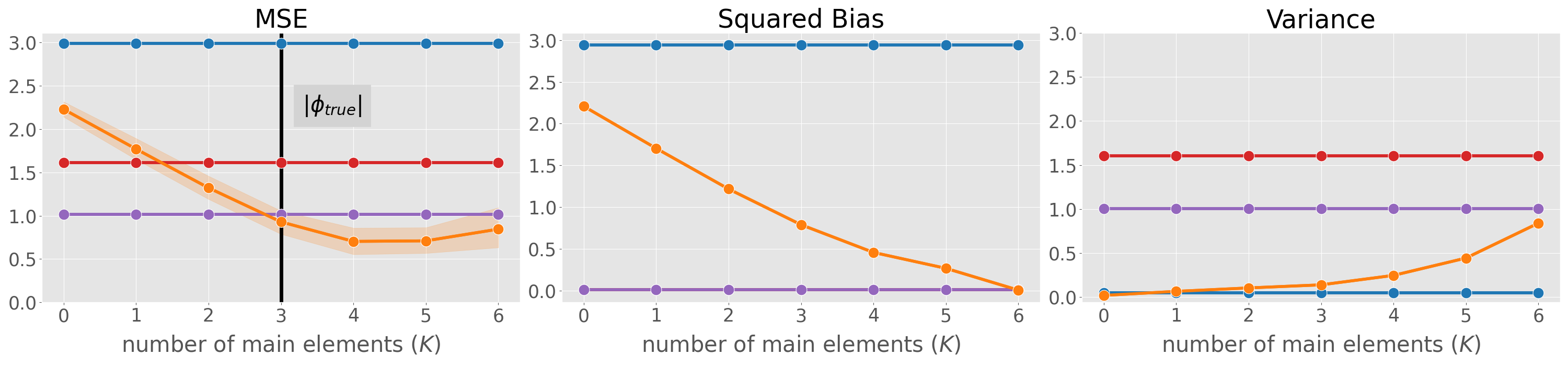}
    \centering \vspace{-3mm}
    \Description{}
    \caption{Comparison of the estimators' MSE, Squared Bias, and Variance with varying numbers of main actions in OPCB ($\hat{\phi}$).}
    \centering
    \label{fig:kth_order}
\end{figure*}

\subsection{Results}
Figures \ref{fig:n_samples} to \ref{fig:kth_order} show the MSE, Squared Bias, and Variance of the estimators computed over 100 simulations with different random seeds to produce synthetic data instances. Unless otherwise specified, the logged data size is set to $n=2000$, the number of actions is $|\calA|=8$, and the ratio of the main effects is $\lambda=0.8$.

\paragraph{\textbf{How does OPCB perform with varying logged data sizes?}} 
Figure \ref{fig:n_samples} varies logged data sizes ($n$) from 500 to 8000. The results demonstrate that OPCB (ours) outperforms the baseline estimators across various logged data sizes by effectively reducing both bias and variance. Specifically, compared to IPS and DR, we observe a significant variance reduction of OPCB (ours), resulting from the use of importance weighting regarding only the main actions $\hat{\phi}(m)$ compared to importance weighting in $\calS$ for IPS and DR. 
The variance reduction is particularly large when the logged data size is small, leading to the reduction of MSE by 92\% compared to IPS and by 86\% compared to DR when $n = 500$, which is quite appealing.
We also observe that OPCB (ours) consistently produces much smaller bias compared to DM. This reduction in bias is attributed to the use of importance weighting to unbiasedly estimate the main effect and two-stage regression in optimizing the regression model of OPCB (ours), which directly deals with the estimation error against the relative reward difference, leveraging the bias analysis.
It would also be intriguing to see that the MSE of OPCB (ours) is somewhat close to OPCB (true) for a range of logged data sizes, suggesting that OPCB (ours) performs reasonably well even without the knowledge about the true reward function, $\phi_{true}$, by leveraging the data-driven selection of the main actions.
Moreover, we can see from the figure that OPCB (best) performs even better than OPCB (ours) and OPCB (true). 
This is an interesting observation because it suggests that the MSE of OPCB is minimized when we do \textbf{NOT} use the true main actions that define the true reward function $\phi_{true}$.
The bias and variance of these estimators observed in Figure \ref{fig:n_samples} tell why this can happen.
Specifically, OPCB (best) produces much smaller bias compared to OPCB (true), by intentionally applying importance weighting to more main actions. 
This would make the residual effect smaller, leading to smaller bias and potentially smaller MSE as long as the variance is well-controlled.

\paragraph{\textbf{How does OPCB perform with varying numbers of unique actions?}}
Next, Figure \ref{fig:n_actions} varies the number of unique actions ($|\calA|$) from 6 to 10, which changes the number of subsets ($|\calS|$) from 64 to 1024. The results show that OPCB (ours) enables more accurate estimation with large subset spaces, while IPS and DR fail in such situations due to their growing variance. This observation suggests that OPCB (ours) becomes even more effective for practical and challenging situations with larger numbers of unique actions. Moreover, the comparison among different versions of OPCB suggests that OPCB (ours) performs similarly to OPCB (true), which uses $\phi_{ture}$, for various numbers of unique actions. It is also true that OPCB (best) performs even better than OPCB (ours) and OPCB (true), an observation consistent with Figure \ref{fig:n_samples}.

\begin{figure*}[t]
    \includegraphics[scale=0.23]{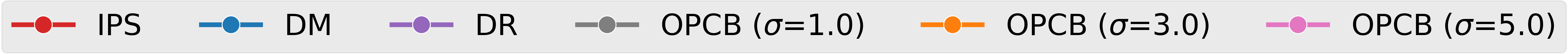} \vspace{3mm}
    \includegraphics[width=0.85\linewidth]{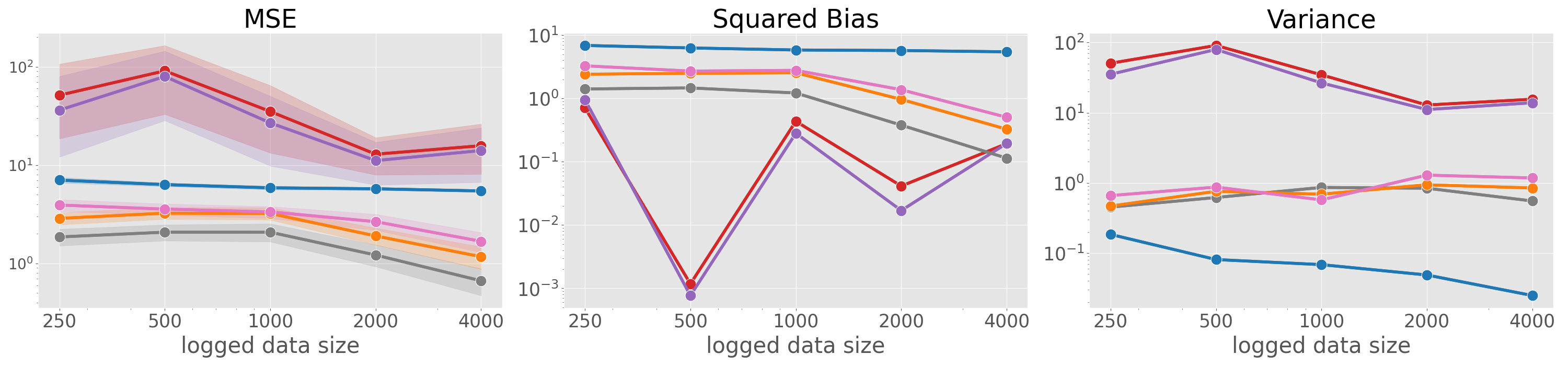}
    \centering \vspace{-3mm}
    \Description{}
    \caption{Comparison of the estimators' MSE, Squared Bias, and Variance with varying logged data sizes ($n$) on KuaiRec.}
    \centering
    \label{fig:kuairec}
\end{figure*}

\paragraph{\textbf{How does OPCB perform with varying ratios of the main effects?}}
Figure \ref{fig:lambda} reports the estimators' performance normalized by that of OPCB (ours) with varying ratios ($\lambda$) of the main effect within the reward function in Eq.~\eqref{expected reward function}. We observe that OPCB (ours) becomes increasingly superior to the baseline estimators as $\lambda$ increases and the main effect becomes a more dominant factor in the reward function. This is because OPCB can unbiasedly estimate a larger part of the reward function via importance weighting, as the main effect becomes more dominant with larger $\lambda$. It is also often true that the pairwise regression in the optimization of the regression model $\hat{f}(x, m)$ is likely to be more accurate when the main effect becomes large because it means that the residual effect becomes relatively small. These results imply that OPCB performs particularly well when there is a small number of key actions that account for the most part of the reward function.

\begin{figure}
  \centering
  \includegraphics[scale=0.22]{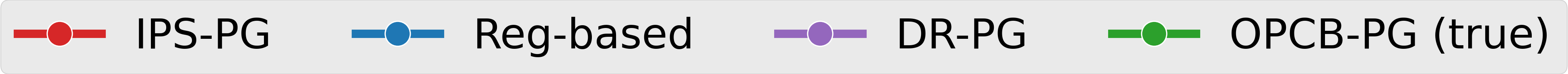} \vspace{3mm}
  \includegraphics[width=\linewidth]{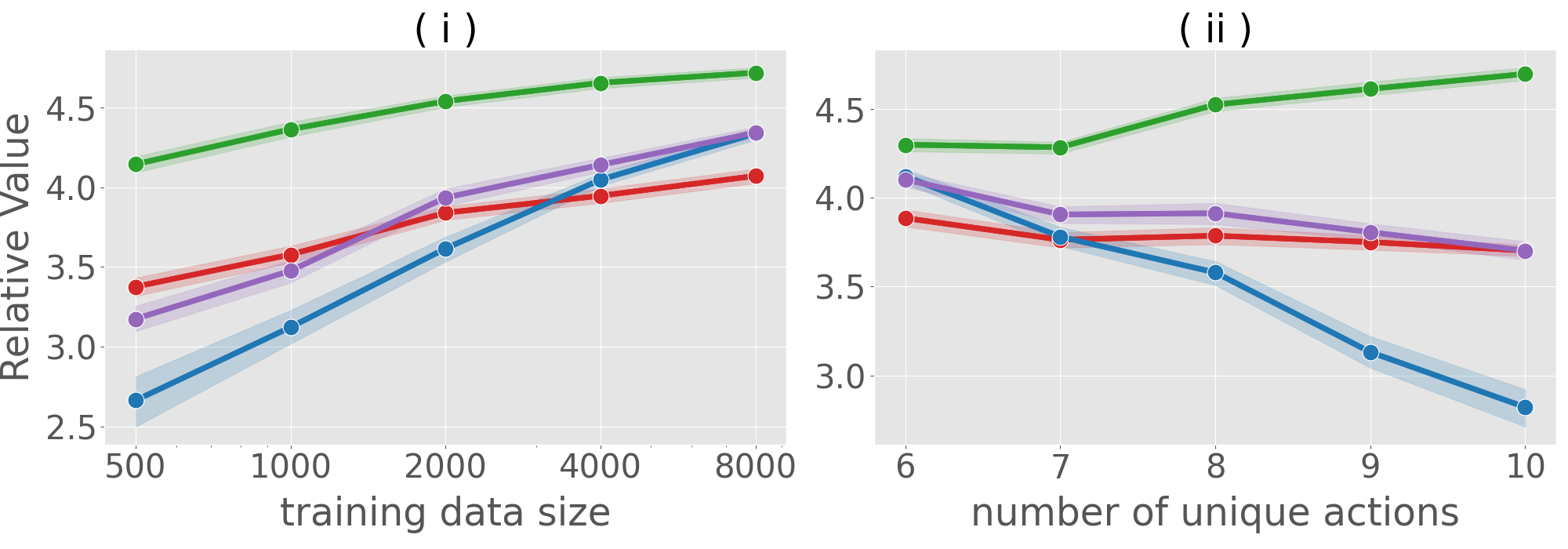} \vspace{-10mm}
  \Description{}
  \caption{Comparison of the policy value (normalized by $V(\pi_0)$ ) of the OPL methods, with varying (i) training data sizes and (ii) the number of unique actions.} \vspace{-2mm}
  \label{fig:opl}
\end{figure}

\paragraph{\textbf{How does OPCB perform with varying numbers of main actions ($|\hat{\phi}(m)|$) used?}}
In Figure~\ref{fig:kth_order}, we investigate the landscape of the bias-variance tradeoff with varying numbers of the main actions used when performing OPCB (ours).
Note that we set $|\phi_{true}(m)|=3$ and $|\calA|=6$, but the number of the main actions used in OPCB (ours) can be different from that of the true reward function in Eq.~\eqref{expected reward function}.
Specifically, in this experiment, we define the set of candidate functions as $\Phi_K := \{\phi \mid |\phi(m)| = K\}$ where $K$ is the number of main actions. We then change the value of $K$ to investigate its effect on the bias and variance of OPCB (ours).
The results demonstrate that the bias of OPCB (ours) decreases and its variance increases as $K$ becomes large, which is consistent with our theoretical analysis. Moreover, the results interestingly suggest that the MSE of OPCB may not necessarily be minimized when $K = |\phi_{true}(m)|$. 
This is because by intentionally using $K$ larger than $|\phi_{true}(m)|$, we can further reduce the bias at the cost of some increased variance, potentially resulting in a smaller MSE. 
This is consistent not only with our bias and variance analysis, but also with the results observed in Figures~\ref{fig:n_samples},~\ref{fig:n_actions}, and~\ref{fig:lambda} where OPCB (best) is more accurate than OPCB (true) by achieving smaller bias.

\paragraph{\textbf{How does OPCB-PG perform compared to the baseline OPL methods?}}
To conduct a synthetic experiment regarding OPL of a CCB policy, we use a neural network with 3 hidden layers to parameterize the policy $\pi_\zeta$. We compare our proposed OPL method, OPCB-PG, against IPS-PG and DR-PG, the policy gradient methods based on the IPS and DR estimators. We also compare the regression-based (Reg-based) baseline, which first estimates the expected reward function as done in DM, and applies the softmax function to the estimated reward function as in Eq.~\eqref{eq:logging} with $\beta = 10$.

Figure~\ref{fig:opl} compares the policy learning effectiveness by the resulting policy values of the OPL methods in the test set (higher the better) with (i) varying training data sizes and (ii) numbers of unique actions. In the left plot, we see that every OPL method performs more similarly with increased training data sizes as expected, but particularly when the logged data size is not large (i.e., $n = 500, 1000$), OPCB-PG outperforms the baselines with a substantial margin. This superior performance of OPCB-PG in the small sample regime can be attributed to the fact that it reduces variance in policy-gradient estimation, which leads to a more sample-efficient OPL in the CCB setup. In the right plot, we can see that OPCB-PG is particularly more effective when the number of unique actions is larger. In particular, when $|\calA| = 10$, OPCB-PG outperforms DR-PG, the second best method, by about 25\%. This observation is akin to what we observed regarding the OPE accuracy, and empirically demonstrates that the lower-variance policy-gradient estimation achieved by OPCB-PG results in a better CCB policy value via OPL.

\section{Real-World Experiment}

This section conducts an OPE experiment on the real-world recommendation dataset called KuaiRec~\citep{gao2022kuairec}, which consists of recommendation logs of the video-sharing app, Kuaishou. Each record in the dataset includes a user ID, recommended video ID as action $a$, and the watch ratio of the recommended video as reward $r \in [0, \infty)$, which represents the play duration divided by the video duration. Each user and video is associated with user and video features, which we consider as context $x$ and features of action $a$, respectively. The key property of KuaiRec is that the user-item interactions are almost fully observed with nearly 100\% density for the subset of its users and items (including 1,411 users and 3,327 items), meaning that the reward function is fully accessible. By leveraging this unique property, we can perform an OPE experiment on this dataset with minimal synthetic component~\citep{gao2022kuairec}.

To perform an OPE experiment on this dataset, we define the expected reward for each action subset as follows.
\begin{align*}
    q(x, m) = \left( \prod_{l=1}^L \tilde{q}(x, m_l) \right)^{(\sum_{l=1}^L \mathbb{I} \{ m_l = a_l \})^{-1}}
\end{align*}
where $\tilde{q}(x, m_l) = \tilde{q}(x, a_l)$ when $m_l = a_l$ and $\tilde{q}(x, m_l) = 1$ when $m_l = \emptyset$. We then sample the reward $r$ from a normal distribution whose mean is $q(x,m)$ and standard deviation $\sigma$ is 3.0.

We define the logging and target policies following Eq.~\eqref{eq:logging} and Eq.~\eqref{eq:evaluation} in the synthetic experiments, but we replace the true expected reward $q(x,m)$ with that estimated by ridge regression and we use $\beta=-0.3$ and $\epsilon = 0.1$ in the real-world experiment. Finally, to simulate a realistic situation where the true MSE is inaccessible regarding the OPCB's data-driven optimization procedure in Eq.~\eqref{eq:tuning}, we vary the estimation errors of the bias term ($\sigma=1.0, 3.0, 5.0$) to sample the noise $\delta_{\phi}$ in Eq.~\eqref{eq:bias_error}. A larger value of $\sigma$ indicates a lower accuracy in estimating the MSE of OPCB to perform Eq.~\eqref{eq:tuning}.

\paragraph{\textbf{Results}}
Figure~\ref{fig:kuairec} compares the MSE, Squared Bias, and Variance of the estimators with varying logged data sizes on KuaiRec. The results demonstrate a similar trend as observed in the synthetic experiment -- OPCB performs consistently better than the baseline methods across various logged data sizes by effectively reducing the bias and variance. It would also be remarkable to see that OPCB brings in a significant reduction in MSE even with the largest noise on its bias estimation when performing the data-driven optimization procedure in Eq.~\eqref{eq:tuning}. These results support the advantage of OPCB in the real CCB problems even with a noisy bias estimation. 

\section{Conclusion}
This paper studied OPE/L for CCB for the first time. We start by identifying the bias and variance issues of the standard approaches, which arise due to the exponential growth of the action subset space. To tackle these issues, we proposed the OPCB estimator based on the formulation of CCB via the factored action space and the corresponding reward function decomposition. Our theoretical analysis highlights the conditions under which OPCB's bias and variance become particularly small. We also discussed a data-driven procedure to optimize the main elements in the action space to identify the appropriate reward decomposition to perform OPCB and extension of OPCB to an OPL method. Experiments on synthetic and real-world datasets demonstrated that OPCB enables far more accurate estimation and policy learning than the baseline methods in a variety of CCB problems.


\bibliographystyle{ACM-Reference-Format}
\bibliography{ref}

\appendix
\newpage
\appendix
\onecolumn

\begin{figure}[t]
    \includegraphics[scale=0.37]{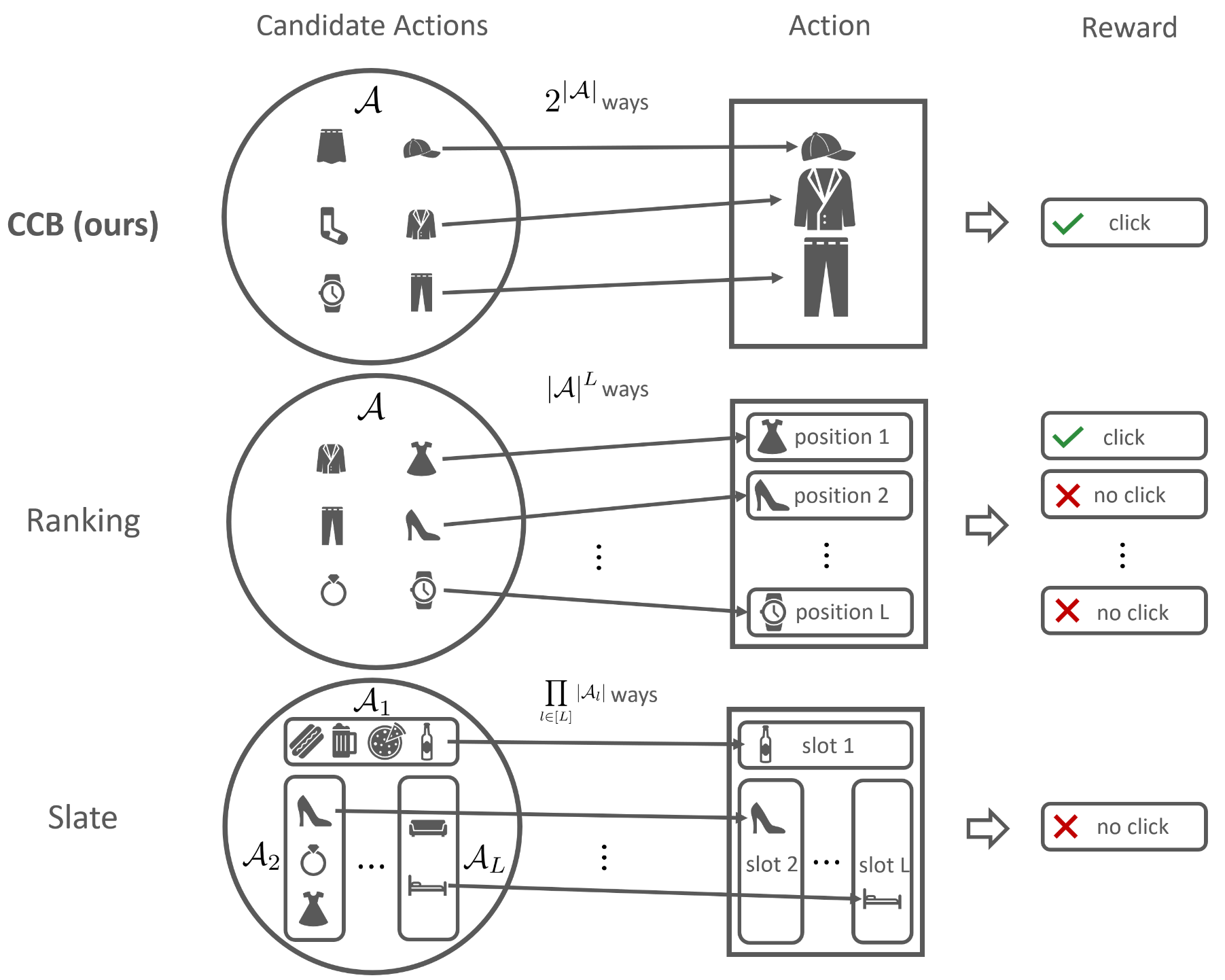}
    \centering
    \Description{}
    \caption{The primary concept of CCB compared to the ranking and slate contextual bandit problems. In terms of the action space, one considers the subset of the candidate actions $\calA$ in CCB, whereas one chooses only one action at each position from the candidate actions $\calA$ in the ranking setting, and one chooses only one action for each slot from each slot-dependent candidate actions $\calA_l$ in the slate setting. Thus, there are $2^{|\calA|}$ action subsets in CCB, which grows exponentially with the number of the candidate actions $|\calA|$, whereas the ranking and slate settings consider $|\calA|^L$ and $\prod_{l \in [L]} |\calA_l|$ actions, respectively, when the length of ranking/slate is $L$. Regarding the observation of rewards, one observes a single reward in CCB and the slate setting, while one observes each reward for each position in the ranking setting.}
    \centering
    \label{fig:concept-comparison-CCB-ranking-slate}
\end{figure}

\section{Extended Related Work}
This section summarizes important related work in detail.

\label{app:related_work}
\subsection{Combinatorial Bandits}
The Combinatorial Multi-Armed Bandits (CMAB) problem, first introduced by \cite{gai2012combinatorial}, is the generalized version of the Multi-Armed Bandits (MAB) where an agent can choose a subset of actions from the action space in each round. In MAB, one aims to find a policy that minimizes regret, a gap between the accumulated expected reward achieved by the optimal policy and the given policy of interest. There are mainly two types of CMAB based on the difference in the reward feedback. The first one is based on semi-bandit feedback~\cite{gai2012combinatorial, chen2013combinatorial, kveton2014matroid, kveton2015tight, combes2015combinatorial, wen2015efficient, wang2018thompson, merlis2019batch, liu2022batch, perrault2020statistical}, in which an agent observes sub-rewards corresponding to each action in the chosen subset of the actions. The second one is full-bandit feedback~\citep{dani2008stochastic, agarwal2021stochastic, rejwan2020top, kuroki2020polynomial}, in which an agent cannot observe the sub-rewards for each action within an action subset but observes a single reward to each action subset. Our work considers the latter full-bandit feedback setting and, in particular, Contextual Combinatorial Bandits (CCB)~\citep{qin2014contextual}, where an agent may choose different action subsets depending on the context.
While CCB has been studied extensively~\citep{qin2014contextual, li2016contextual, chen2018contextual, takemura2021near, zierahn2023nonstochastic}, 
our paper is distinct from the existing literature in that we study Off-Policy Evaluation and Learning (OPE/L) for the first time in the CCB literature. 
While conventional online learning algorithms require extensive exploration particularly when the number of candidate actions is large, our offline approach enables more safe and costless policy evaluation and learning by leveraging logged data.

\subsection{Off-Policy Evaluation and Learning}
Off-Policy Evaluation and Learning (OPE/L) \cite{dudik2014doubly, wang2017optimal, liu2018breaking, farajtabar2018more, su2019cab, su2020doubly, kallus2020optimal, metelli2021subgaussian, saito2022off, saito2023off} 
is known as a safe and ethical alternative to online A/B tests and online learning, as OPE/L aims to evaluate or learn a new policy solely from offline logged data.

In OPE, there are mainly three approaches: model-based, model-free, and hybrid. A typical model-based approach, referred to as Direct Method (DM) \cite{beygelzimer2009offset}, estimates the policy value using the estimated expected reward given context and action ($\hat{q}(x,a)$), learned by an off-the-shell supervised machine learning method. DM has low variance but may incur considerable bias when the estimated reward is inaccurate, which is often the case under partial reward, and covariate shift arises from logged data. 
In contrast, the model-free approach called Inverse Propensity Scoring (IPS) \cite{horvitz1952generalization} re-weighs the observed reward by the density ratio of actions between the logging and target policies. IPS achieves unbiasedness and consistency under common support and unconfoundedness. However, IPS is susceptible to a significant variance caused by a large action space. Although there are some techniques to reduce the scale of importance weight, such as clipping \cite{swaminathan2015counterfactual, su2019cab,su2020doubly} and normalizing \cite{swaminathan2015self}, these transformation incurs nonnegligible bias instead. Doubly Robust (DR) \cite{dudik2011doubly} is a hybrid approach that uses the estimated reward as a control variable and applies importance sampling only on the residual. DR is unbiased either when IPS or DM is unbiased, and also reduces the variance of IPS due to the use of a control variate. However, the variance reduction is limited since DR still uses the large importance weight. Thus, DR suffers from significant variance when encountering large importance weights due to the exponential number of subsets to be considered in the CCB setting.

To deal with large action spaces in the \textit{single} action ($a \in \calA$) setting, \cite{saito2022off, saito2023off} uses auxiliary information, such as action embeddings, in OPE. In particular, Marginalized Inverse Propensity Scoring (MIPS) \cite{saito2022off} applies importance sampling on the action embedding space ($\mathcal{E}$):
\begin{align*}
    \mips := \meanN \frac{p(e_i|x_i, \pi)}{p(e_i|x_i, \pi_0)} r_i = \meanN w(x_i, e_i) r_i, 
\end{align*}
where $e \in \calE \subseteq \mR^{d_e}$ is an action embedding, sampled from a conditional distribution $p(e|x, a)$. $p(e|x, \pi) := \sumA \pi(a|x) p(e|x, a)$ is the probability of choosing action embedding $e$ given context $x$ induced by policy $\pi$, and $w(x, e) = p(e|x, \pi) / p(e|x, \pi_0)$ is the marginalized importance weight. 
By using the marginalized importance weight, we can expect a significant variance reduction as $|\mathcal{E}|$ becomes small compared to the original action space of $|\mathcal{A}|$.
However, the unbiasedness of MIPS depends on the satisfaction of the \textit{no direct effect} assumption, which requires that action embedding $e$ mediates all the effect of action $a$ on reward $r$ (i.e., $a \indep r |x, e$). Thus, when we naively apply MIPS to the CCB settings, the action subset's embedding can be the concatenating of action embeddings, which can be high-dimensional, leading to a high variance issue as the original IPS has. 
To tackle the issue associated with high-dimensional embeddings, Off-policy evaluation estimator based on the Conjunct Effect Model (OffCEM) \cite{saito2023off} defines the estimator leveraging the clustered action space:
\begin{align*}
    \offcem := \meanN \left\{ w(x_i, \phi(x_i, a_i)) (r_i - \hat{f}(x_i, a_i)) + \mE_{\pi(a|x_i)} [\hat{f}(x_i, a)] \right\}, 
\end{align*}
where $\phi: \calX \times \calA \to \calC$ is a clustering function, $w(x, c) := \pi(c|x)/\pi_0(c|x)$ is the importance weight in the cluster space $\calC$, $\pi(c|x) := \sumA \pi(a|x) \ind \{ \phi(x, a) = c \}$, and $\hat{f}: \calX \times \calA \to \mR$ is the estimator of the expected reward function. 
By applying cluster-wise importance weight, OffCEM reduces the variance of naive IPS. Moreover, OffCEM mitigates the bias of MIPS caused by the no-direct assumption, by introducing the regression estimation about the residual effect. In particular, OffCEM is unbiased under the local correctness ($\phi(x, a) = \phi(x, b) \implies \Delta_q(x, a, b) = \Delta_{\hat{f}}(x, a, b) \quad \forall x \in \calX, a, b \in \calA$) holds. 
In this paper, we consider a way to apply OffCEM in the CCB setting, where clustering of action subsets is often challenging. Specifically, we cluster action subsets by identifying important (main) actions, leveraging the newly introduced factored action space in CCB. Introducing the main actions in action subsets enables us to interpret the clusters of action subsets readily. Moreover, we provided a detailed data-driven procedure to obtain the main actions in a given action subset so that we can minimize the MSE of the proposed estimator in Section~\ref{subsec:hyperparameter-tuning} with empirical guarantees on the performance of the optimization procedure in Section~\ref{subsec:synthetic ope}. Furthermore, we extended the OPCB estimator to the OPL in CCB to learn a better policy where \cite{saito2023off} only considered OPE in a single action space.

In OPL, there are two main approaches: regression-based approaches and policy gradient approaches. Regression-based approaches estimate the expected reward and then use it to define a new policy, such as the epsilon-greedy or softmax policies. Like DM in OPE, regression-based approaches often suffer from bias when the regression is inaccurate. In contrast, policy gradient approaches estimate the gradient of the parameterized policy value and use it to update the parameter of the policy via the iterative gradient ascent. IPS is the typical choice for the estimator of the policy gradient but it suffers from high variance. Our approach effectively reduces the variance of the policy gradient estimation by leveraging the OPCB estimator in CCB.

\subsection{Off-Policy Evaluation for Ranking and Slate Policies}
Ranking and slate settings are similar settings to CCB. In OPE for ranking policies, we consider the ordered list of the candidate items as the action, where the number of elements in the ordered list is predefined, and we can observe the position-wise rewards corresponding to each sub-action. Most of the estimators for ranking OPE ~\citep{li2018offline, mcinerney2020counterfactual, kiyohara2022doubly, kiyohara2023off} are based on assumptions about how the users behave towards the ranked items, such as independently or from top to bottom. However, when we consider a combination of the candidate actions, the order of the chosen actions in an action subset does not matter, and we cannot observe the corresponding sub-reward to each unique action in an action subset. Thus, the existing estimators in ranking OPE are not applicable to evaluating action subsets.

In OPE for slate bandit policies, we aim to evaluate the slate policies where the slate consists of $L$ slot actions, whose action set $\mathcal{A}_l$ may differ across different slots (e.g., in online ads, $\mathcal{A}_1$ can be the candidate set of slogans, $\mathcal{A}_2$ can be that of key visuals). Similar to the CCB settings, we cannot observe the slot-level rewards. To deal with the lack of observations of sub-rewards, PseudoInverse (PI) estimators~\citep{swaminathan2017off, vlassis2021control} assume that the total reward is linearly decomposable to sub-rewards corresponding to each slot (called \textit{linearity} assumption). This reduces the importance weight from the product to the sum of slot-wise importance weight. However, PI incurs considerable bias when the linearity assumption often does not hold, which is often the case in practice due to the interaction effects among the elements in a slate. Another estimator called Latent IPS (LIPS)~\citep{kiyohara2024slate}, which employs importance weight defined in a latent slate abstract space to mitigate the estimator's variance. However, it completely disregards the effect of the slate that the latent abstraction cannot capture, while our estimator is able to take the residual effect into account via regression.

\section{Connection and Comparison to Slate OPE Estimators} \label{app:slate_ope} 
In this section, we further explore the connection and comparison of CCB settings to slate contextual bandit including formal definitions of the estimators in slate contextual bandits settings.
Recall that, in CCB settings, we aim to evaluate a target policy $\pi: \calX \to \Delta(\calS)$ that selects a subset $s \in \calS = 2^{\calA}$ of the action space $\calA$. This problem can be reformulated by introducing a binary indicator, $m_l \in \calM_l = \{ \emptyset, a_l \}$, that represents if each candidate action $a_l \in \calA$ (e.g., bed, drawer, mirror) is in the selected subset $s$ (i.e., $a_l \in s$ or $a_l \notin s$). This reformulation allows us to consider OPE with factored action space $\calM := \prod_{l \in [L]} \calM_l = \prod_{l \in [L]} \{ \emptyset, a_l \}$, which is instrumental in focusing on the specific primary items in the chosen set $s$, instead of an action subset space $\calS$. In contrast, in slate OPE, we aim to evaluate a target policy $\pi: \calX \to \Delta(\prod_{l \in [L]}A_l)$ that chooses a list of actions $a = (a_1, \cdots, a_L) \in \prod_{l \in [L]} \calA_l$ where we always choose one action $a_l \in \calA_l$ for each slot $l$ from the corresponding action set $\calA_l := \{ a_{l,1}, a_{l,2}, \cdots, a_{l,|\calA_l|} \}$. 

To bridge the gap between these two distinct setups, we can indeed come up with a more general formulation that includes them as special cases by introducing a more general action space $\tilde{\calM} := \prod_{l \in [L]} \tilde{\calM}_l$ where, in each action space $\tilde{\calM}_l$, the agent can choose $\tilde{m}_l \in \tilde{\calM}_l = \{ \emptyset \} \cup \calA_l$, which is either an empty set $\emptyset$ or the element $a_l \in \calA_l$. This allows us to consider the two-step decision process to decide what to include in a recommendation, e.g.,  in online ads, we first decide (1) which component (e.g., title, discount rate, slogan, etc.) to include in the ads, and then decide (2) which title to use (e.g., title 1,  title 2, etc.) if the title is included in the combination. Under this general formulation of CCB and slate contextual bandits, the OPCB estimator in Eq.~\eqref{eq:K-th-order-OPCB} can readily be applied to this generalized setting. The following paragraphs cover the comparison and connection of OPCB and two primary estimators for OPE in slate contextual bandits under the newly introduced general formulation of CCB and slate settings.

\paragraph{PseudoInverse (PI) estimator} 
The PI estimator~\citep{swaminathan2017off, vlassis2021control} considers applying the element-wise importance weight, i.e., $w(x, \tilde{m}_l) = \pi(\tilde{m}_l|x) / \pi_0(\tilde{m}_l|x)$ to estimate the policy value as follows.
\begin{align*}
    \PI := \meanN \left( \sumL w(x_i, \tilde{m}_{i, l}) - L + 1 \right) r_i, 
\end{align*}
PI reduces the variance of IPS, as the importance weight is reduced from $\left( \prod_{l=1}^L w(x, \tilde{m}_l) \right)$ to $\left( \sum_{l=1}^L w(x, \tilde{m}_l) \right)$. PI is also unbiased when the expected reward is linearly decomposable as $q(x, \tilde{m}) = \sum_{l=1}^L \phi_l(x, \tilde{m}_l)$, where $\{ \phi_l \}_{l=1}^L$ is some (intrinsic) element-wise reward function. However, this \textit{linearity} assumption is often unrealistic in the combinatorial-slate setting because the reward always increases if the chosen elements increase. For example, in interior design, having too much furniture in a bedroom can negatively affect the reward. However, this linearity assumption implies that using all the elements is always the best. The proposed OPCB, in contrast, avoids such an unrealistic linearity assumption by taking the interactions among candidate actions into account.

To further show an interesting connection between OPCB and PI, let us introduce the following 1-main-element OPCB estimator:
\begin{align*}
    \lopca := \meanN \Bigg\{ w(x_i, \tilde{m}_{i, l}) \left( r_i - \hat{f}_l(x_i, \tilde{m}_i) \right) + \mE_{\pi(\tilde{m}|x_i)} \left[ \hat{f}_l(x_i, \tilde{m}) \right] \Bigg\},
\end{align*}
where the \textit{main effect} $g(x, \phi(\tilde{m})) = g(x, \tilde{m}_l)$ corresponds to the element-wise expected reward of candidate action $\tilde{m}_l$ and $\hat{f}_l(\cdot, \cdot)$ considers the residual effect $h(x, \tilde{m})$, including the interaction with other actions. Therefore, by taking the average of $\hat{V}^l_{\text{OPCB}}$ over $l \in [L]$, we can reproduce the linearity structure of PI while also taking the interaction as $q(x, \tilde{m}) = \sum_{l=1}^L \phi(x, \tilde{m}_l) + H(x, \tilde{m})$, where $H(x, \tilde{m})$ is the interaction effect. This estimator is strictly more general than PI and overcomes the limitation of introducing linearity assumption.
We call this estimator $\hat{V}^{\text{OPCB-PI}}(\pi; \calD) := \sum_{l=1}^L \hat{V}^l_{\text{OPCB}}(\pi; \calD) / L$ as the OPCB-PI estimator.

\paragraph{Latent IPS (LIPS) estimator}
LIPS~\citep{kiyohara2024slate} represents another estimator for slate OPE, utilizing the importance weight in a latent slate abstract space to alleviate the high variance.
\begin{align*}
    \lips := \meanN w \left( x_i, \phi(\tilde{m}_i) \right)  r_i, 
\end{align*}
where $\phi: \tilde{\calM} \to \calZ$ maps elements from the general action space $\tilde{\calM}$ to the latent space $\calZ$ and $ w(x, \phi(\tilde{m})) := \pi(\phi(\tilde{m})|x) / \pi_0(\phi(\tilde{m})|x)$ is the latent importance weight. Although the low variance of LIPS is favorable, it is not always guaranteed that we can obtain a good abstraction of a slate. Moreover, interpreting latent space poses a challenge, particularly in practical applications. In contrast, the OPCB estimator constantly achieves low variance and provides a clear interpretation of evaluation due to the decomposition of the expected reward. Furthermore, the OPCB estimator captures the residual effect $h(x, \tilde{m})$ that LIPS ignored, further reducing bias. Hence, OPCB overcomes the limitations of not only conventional OPE estimators but also those specific to slate OPE.

\section{Theoretical Analysis of the Policy Gradient Estimated by the OPCB Estimator}
\label{app-theoretical-analysis-OPCB-PG}
In this section, we provide a detailed theoretical analysis of the gradient of the OPCB estimator used for OPL in CCB proposed in Section \ref{sec:opl}. Overall, the bias and variance of the gradient of the OPCB estimator are similar to those of the OPCB estimator for OPE. Before providing the theoretical properties of the gradient of the OPCB estimator, we introduce an additional condition required for the unbiasedness of the estimator.
\begin{assumption}[Full Support w.r.t the Main Actions]
    \label{ass.full-support-main-actions}
    The logging policy $\pi_0$ is said to satisfy full support w.r.t the set of the main actions $\phi(m)$ if $\pi_0(\phi(m)\,|\,x)> 0$ for all $m \in \calM$ and $x \in \calX$. 
\end{assumption}
Then under the newly introduced condition, we provide the unbiasedness of the gradient of the OPCB estimator in the following theorem.
\begin{theorem}[Unbiasedness of the gradient of the OPCB estimator]
    \label{thm-unbiasedness-OPCB-PG}
    Under Conditions \ref{ass.full-support-main-actions} and \ref{ass.local_correctness_Kth}, the gradient of the OPCB estimator is unbiased. i.e., $\mE_{\calD}[\nabla_{\zeta} \Kopcazeta] = \nabla_{\zeta} V(\pi_{\zeta})$. See Appendix \ref{proof-unbiasedness-OPCB-PG} for the proof.
\end{theorem}
Theorem \ref{thm-unbiasedness-OPCB-PG} suggests that the gradient of the OPCB estimator is unbiased if the OPCB estimator for OPE is unbiased and full support for the logging policy is satisfied. Beyond the unbiasedness, when Condition \ref{ass.local_correctness_Kth} is violated, the gradient of the OPCB estimator incurs the following bias.
\begin{theorem}[Bias of the gradient of the OPCB estimator]    
    \label{thm-bias-OPCB-PG}
    Under Condition \ref{ass.full-support-main-actions}, the gradient of the OPCB estimator has the following bias.
    \begin{align}
        \bias\left( \nabla_{\zeta} \Kopcazeta \right) 
        &= \mE_{p(x) \pi_{\zeta}(\phi(\Bar{m})|x)} \Bigg[ \sum_{\substack{m < m':\\ \phi(m) = \phi(m') = \phi(\Bar{m})}} \underbrace{\pi_0(m|x, \phi(\Bar{m})) \pi_0(m'|x, \phi(\Bar{m}))}_{\text{(i)}} \notag \\
        \times & \underbrace{\Bigg( \frac{\pi_{\zeta}(m'|x, \phi(\Bar{m})) }{\pi_0(m'|x, \phi(\Bar{m}))} \nabla_{\zeta} \log \pi_{\zeta}(m', \phi(\Bar{m})|x)  - \frac{\pi_{\zeta}(m|x, \phi(\Bar{m}))}{\pi_0(m|x, \phi(\Bar{m}))} \nabla_{\zeta} \log \pi_{\zeta}(m, \phi(\Bar{m})|x)  \Bigg)}_{\text{(ii)}}
        \underbrace{\left( \Delta_q(x, m, m') - \DeltafKth \right)}_{\text{(iii)}}  \Bigg] \notag , 
    \end{align}
    where $\pi_{\zeta}(m\,|\,x, \phi(\Bar{m})) = \pi_{\zeta}(m\,|\,x) \ind \{\phi(m) = \phi(\Bar{m})\} / \pi_{\zeta}(\phi(\Bar{m})\,|\,x)$, $\pi_{\zeta}(m, \phi(\Bar{m})|x) = \pi_{\zeta}(\phi(\Bar{m})|x) \pi_{\zeta}(m|x, \phi(\Bar{m}))$ and $m, m' \in \calM$. See Appendix \ref{proof-bias-OPCB-PG} for the proof.
\end{theorem}
The first and third terms of the bias of the gradient of the OPCB estimator are identical to the ones of the bias of the OPCB estimator in Section \ref{sec:theoretical-analysis}. The only difference between the bias terms of the gradient of the OPCB estimator and the OPCB estimator lies in the second term (ii), where the second term of the bias of the gradient of the OPCB estimator has additional term $\nabla_{\zeta} \log \pi_{\zeta}(m, \phi(\Bar{m})|x)$. Lastly, in contrast to the bias analysis, we provide the variance of the gradient of the OPCB estimator as follows.
\begin{theorem}[Variance of the gradient of the OPCB estimator]
    \label{thm-variance-OPCB-PG}
    Under Conditions \ref{ass.full-support-main-actions} and \ref{ass.local_correctness_Kth}, the variance of the OPCB estimator is as follows.
    \begin{align}
        n \mV_{\calD} \left[ \nabla_{\zeta} \Kopcazeta \right] &= \mE_{p(x) \pi_0(m|x)} \left[\left(\frac{\pi_{\zeta}(\phi(m)|x)}{\pi_0(\phi(m)|x)}\right)^2 \chi_{\zeta}(x, \phi(m)) \sigma^2(x, m)\right] \notag \\
        +& \mE_{p(x)} \left[ \mV_{\pi_0(m|x)} \left[\frac{\pi_{\zeta}(\phi(m)|x)}{\pi_0(\phi(m)|x)} \DeltaqfKth s_{\zeta}(x, \phi(m)) \right] \right] 
        + \mV_{p(x)} \left[ \mE_{\pi_{\zeta}(m|x)} \left[q(x, m) s_{\zeta}(x, \phi(m)) \right] \right] \label{eq:variance-OPCB-PG}, 
    \end{align}
    where $\DeltaqfKth := q(x, m) - \fKth(x, m)$ is the estimation error of $\fKth(x, m)$ against the expected reward function $q(x, m)$, $s_{\zeta}(x, \phi(m)) := \nabla_{\zeta} \log \pi_{\zeta}(\phi(m)|x)$, and $\chi_{\zeta}^{(j)}(x, \phi(m)) := s_{\zeta}^{(j)}(x, \phi(m))^2$ where $\chi_{\zeta}^{(j)}(x, \phi(m))$ and $s_{\zeta}^{(j)}(x, \phi(m))$ are the $j$-th elements of the vector $\chi_{\zeta}(x, \phi(m))$ and $s_{\zeta}(x, \phi(m))$, respectively. See Appendix \ref{proof-variance-OPCB-PG} for the proof.
\end{theorem}
Theorem \ref{thm-variance-OPCB-PG} demonstrates that the variance of the gradient of the OPCB estimator is dependent on the importance weight on the main elements $\pi_{\zeta}(\phi(m)|x) / \pi_0(\phi(m)|x)$ similar to the variance of the OPCB estimator in Section \ref{sec:theoretical-analysis}, which is instrumental in the variance reduction compared to to the gradient of the IPS and DR.

\section{Omitted Proofs}
In this section, we provide proof of the proposed estimator's unbiasedness, bias, and variance. Before starting the proof, we first introduce the different interpretations of Condition~\ref{ass.local_correctness_Kth}, which will be used in the proof. Recall that Conditional Pairwise Correctness (Condition~\ref{ass.local_correctness_Kth}) requires the following:
\begin{align}
    \phi(m) = \phi(m') \implies \Delta_{q}(x, m, m') = \Delta_{\hat{f}}(x, m, m') \quad \forall x \in \calX, m, m' \in \calM \label{eq:conditional-piecewise-correctness}, 
\end{align}
The condition \eqref{eq:conditional-piecewise-correctness} means that the regression model $\hat{f}$ should preserve the relative value difference of the expected rewards between $m$ and $m'$ if the set of the main actions $\phi(m)$ and $\phi(m')$ are identical. Then the condition \eqref{eq:conditional-piecewise-correctness} implies the following:
\begin{align}
    \phi(m) = \phi(m') \implies \Delta_{q, \hat{f}}(x, m) = \Delta_{q, \hat{f}}(x, m') \quad \forall x \in \calX, m, m' \in \calM \label{eq:conditional-piecewise-correctness-2}, 
\end{align}
The condition \eqref{eq:conditional-piecewise-correctness-2} means that if the set of the main action $\phi(m)$ and $\phi(m')$ are identical, then the difference of the estimation errors of the regression model $\hat{f}$ between $m$ and $m'$ should be the same. This condition implies the existence of some function $\Bar{\Delta}: \calX \times \calM \to \mR$ such that
\begin{align}
    \Delta_{q, \hat{f}}(x, m) = \Bar{\Delta}(x, \phi(m)) \quad \forall x \in \calX, m, \in \calM \label{eq:conditional-piecewise-correctness-3}. 
\end{align}
We will use the condition \eqref{eq:conditional-piecewise-correctness-3} when we assume conditional piecewise correctness to prove the unbiasedness and variance of the proposed estimator.

\subsection{Proof of Theorem \ref{thm.unbiasedness_OPCB_Kth}}
\label{proof_unbiasedness_OPCB_Kth}
\begin{proof}
    We can show the unbiasedness of the OPCB estimator under Conditions \ref{ass.common_support_Kth} and \ref{ass.local_correctness_Kth} below.
    \begin{align*}
        & \mE_{\calD} \left[ \Kopca \right] \\
        ={}& \mE_{p(x_i) \pi_0(m_i|x_i) p(r_i|x_i, m_i) \forall i \in [n]} \left[ \meanN \left\{ \frac{\piiKth}{\piizeroKth} \left( r_i - \fKth(x_i, m_i) \right) + \mE_{\pi(m|x_i)}[\fKth(x_i, m)] \right\} \right] \\
        ={}& \meanN \mE_{p(x_i) \pi_0(m_i|x_i) p(r_i|x_i, m_i)} \left[  \frac{\piiKth}{\piizeroKth} \left( r_i - \fKth(x_i, m_i) \right) + \mE_{\pi(m|x_i)}[\fKth(x_i, m)]  \right] \quad \because \text{linearity of $\mE$}  \\
        ={}&  \mE_{p(x) \pi_0(m|x) p(r|x, m)} \left[ \frac{\piKth}{\pizeroKth} \left( r - \fKth(x, m) \right) + \mE_{\pi(m'|x)}[\fKth(x, m')] \right] \quad \because \text{i.i.d. assumption}  \\
        ={}&  \mE_{p(x) \pi_0(m|x)} \left[ \frac{\piKth}{\pizeroKth} \left( q(x, m) - \fKth(x, m) \right) + \mE_{\pi(m'|x)}[\fKth(x, m')]\right] \quad \because \text{definition of $q(x, m)$} \\
        ={}&  \mE_{p(x)} \left[ \sumM \pi_0(m|x) \frac{\piKth}{\pizeroKth} \left( q(x, m) - \fKth(x, m) \right) + \mE_{\pi(m'|x)}[\fKth(x, m')] \right] \\
        ={}&  \mE_{p(x)} \left[ \sumM \pi_0(m|x) \frac{\piKth)}{\pizeroKth} \Bar{\Delta}(x, \phi(m)) + \mE_{\pi(m'|x)}[\fKth(x, m')] \right] \quad \because \text{conditional piecewise correctness}  \\
        ={}&  \mE_{p(x)} \left[ \sumM \pi_0(m|x) \sum_{\phi(m') \in 2^{\calA}} \ind \{ \phi(m') = \phi(m) \} \frac{\piprimeKth}{\piprimezeroKth} \Bar{\Delta}(x, \phi(m')) + \mE_{\pi(m'|x)}[\fKth(x, m')] \right] \\
        ={}&  \mE_{p(x)} \left[ \sum_{\phi(m') \in 2^{\calA}} \frac{\piprimeKth}{\piprimezeroKth} \Bar{\Delta}(x, \phi(m')) \sumM \pi_0(m|x) \ind \{ \phi(m') = \phi(m) \} + \mE_{\pi(m'|x)}[\fKth(x, m')] \right] \\
        ={}&  \mE_{p(x)} \left[ \sum_{\phi(m') \in 2^{\calA}} \frac{\piprimeKth}{\piprimezeroKth} \Bar{\Delta}(x, \phi(m')) \piprimezeroKth + \mE_{\pi(m'|x)}[\fKth(x, m')] \right] \quad \because \text{definition of  $\piprimezeroKth$} \\
        ={}&  \mE_{p(x)} \left[ \sum_{\phi(m') \in 2^{\calA}} \piprimeKth \Bar{\Delta}(x, \phi(m')) + \mE_{\pi(m'|x)}[\fKth(x, m')] \right] \quad \because \text{cancel out $\piprimezeroKth)$} \\
        ={}&  \mE_{p(x)} \left[ \sum_{\phi(m') \in 2^{\calA}} \Bar{\Delta}(x, \phi(m')) \sumM \pi(m|x) \ind \{ \phi(m) = \phi(m') \} + \mE_{\pi(m'|x)}[\fKth(x, m')] \right] \\
        ={}& \mE_{p(x)} \left[ \sumM \pi(m|x)  \sum_{\phi(m') \in 2^{\calA}} \Bar{\Delta}(x, \phi(m')) \ind \{ \phi(m) = \phi(m') \} + \mE_{\pi(m'|x)}[\fKth(x, m')] \right] \\
        ={}& \mE_{p(x)} \left[ \sumM \pi(m|x)   \Bar{\Delta}(x, \phi(m))  + \mE_{\pi(m'|x)}[\fKth(x, m')] \right] \\
        ={}& \mE_{p(x) \pi(m|x)} \left[ q(x, m) \right] \quad \because \text{conditional piecewise correctness} \\
        ={}& V(\pi).
    \end{align*}
\end{proof}

\subsection{Proof of Theorem \ref{thm.bias_OPCB_Kth}}
\label{proof_bias_OPCB_Kth}
\begin{proof}
    We derive the bias of the OPCB estimator under Condition \ref{ass.common_support_Kth} below by using Lemma \ref{lemma.a.1_Kth} and \ref{lemma.a.2_Kth}.
    \begin{align}
        & \bias \left( \Kopca \right) \notag \\
        ={}& \mE_{p(x_i) \pi_0(m_i|x_i) p(r_i|x_i, m_i) \forall i \in [n]} \left[ \meanN \left\{ \frac{\piiKth}{\piizeroKth} \left( r_i - \fKth(x_i, m_i) \right) + \mE_{\pi(m|x_i)}[\fKth(x_i, m)] \right\} \right] - V(\pi) \notag \\
        ={}& \mE_{p(x) \pi_0(m|x) p(r|x, m)} \left[ \frac{\piKth}{\pizeroKth} \left( r - \fKth(x, m) \right) + \mE_{\pi(m'|x)}[\fKth(x, m')] \right] - V(\pi) \notag \quad \because \text{i.i.d. assumption} \\
        ={}& \mE_{p(x) \pi_0(m|x)} \left[ w(x, \phi(m)) \DeltaqfKth + \mE_{\pi(m'|x)}[\fKth(x, m')] \right]  - \mE_{p(x) \pi(m|x)} \left[ q(x, m) \right] \notag \\
        ={}& \mE_{p(x)} \left[ \sumM \pi_0(m|x) w(x, \phi(m)) \DeltaqfKth \right] \notag + \mE_{p(x)}\left[ \sumM \pi(m|x) \fKth(x, m)\right]  - \mE_{p(x)} \left[ \sumM \pi(m|x) q(x, m) \right] \notag \\
        ={}& \mE_{p(x)} \left[ \sumM \pi_0(m|x) \DeltaqfKth w(x, \phi(m)) \right] - \mE_{p(x)}\left[ \sumM \pi(m|x) \DeltaqfKth \right] \notag \\
        ={}& \mE_{p(x) \piBarzeroKth} \Bigg[ \sumM w(x, m) \pi_0(m|x, \phi(\Bar{m})) \notag \\
        & \times \left( \left( \sumMprime \pi_0(m'|x, \phi(\Bar{m})) \DeltaqfprimeKth \right) -\DeltaqfKth \right) \Bigg] \notag \quad \because \text{Lemmas \ref{lemma.a.1_Kth} and \ref{lemma.a.2_Kth}} \\
        ={}& \mE_{p(x) \piBarzeroKth} \Bigg[ \sum_{m < m'} \pi_0(m|x, \phi(\Bar{m})) \pi_0(m'|x, \phi(\Bar{m})) \left( \Delta_q(x, m, m') - \Delta_{\fKth}(x, m, m') \right) \left( w(x, m') - w(x, m) \right) \Bigg] \label{eq:MIPS_technique} \\
        ={}& \mE_{p(x) \piBarzeroKth} \Bigg[ \sum_{m < m'} \frac{\pi_0(m|x) \ind\{ \phi(m) = \phi(\Bar{m}) \}}{\piBarzeroKth} \frac{\pi_0(m'|x) \ind\{ \phi(m') = \phi(\Bar{m}) \}}{\piBarzeroKth} \notag \\
        & \times \left( \Delta q(x, m, m') - \Delta_{\fKth}(x, m, m') \right) \left( w(x, m') - w(x, m) \right) \Bigg] \notag \\
        ={}& \mE_{p(x) \piBarzeroKth} \Bigg[ w(x, \phi(\Bar{m})) \sum_{\substack{m < m':\\ \phi(m) = \phi(m') = \phi(\Bar{m})}} \pi_0(m|x, \phi(\Bar{m})) \pi_0(m'|x, \phi(\Bar{m})) \notag \\
        & \times \left( \Delta_q(x, m, m') - \Delta_{\fKth}(x, m, m') \right) \left( \frac{\pi(m'|x, \phi(\Bar{m}))}{\pi_0(m'|x, \phi(\Bar{m}))} - \frac{\pi(m|x, \phi(\Bar{m}))}{\pi_0(m|x, \phi(\Bar{m}))} \right) \Bigg] \label{eq:weight_modification} \\
        ={}& \mE_{p(x) \piBarKth} \Bigg[ \sum_{\substack{m < m':\\ \phi(m) = \phi(m') = \phi(\Bar{m})}} \pi_0(m|x, \phi(\Bar{m})) \pi_0(m'|x, \phi(\Bar{m})) \notag \\
        & \times \left( \Delta_q(x, m, m') - \Delta_{\fKth}(x, m, m') \right) \left( \frac{\pi(m'|x, \phi(\Bar{m}))}{\pi_0(m'|x, \phi(\Bar{m}))} - \frac{\pi(m|x, \phi(\Bar{m}))}{\pi_0(m|x, \phi(\Bar{m}))} \right) \Bigg] , \notag 
    \end{align}
    where we used Lemma B.1 of \cite{saito2022off} at equation \eqref{eq:MIPS_technique} given an arbitrary relational operator between two different action subsets $m, m' \in \calM$ and we used 
    \begin{align*}
        w(x, m) = \frac{\pi(m|x)}{\pi_0(m|x)} = \frac{\pi(m|x, \phi(\Bar{m})) \piBarKth}{\pi_0(m|x, \phi(\Bar{m})) \piBarzeroKth} = \frac{\pi(m|x, \phi(\Bar{m}))}{\pi_0(m|x, \phi(\Bar{m}))} w(x, \phi(\Bar{m})), 
    \end{align*}
    for all $\phi(m) = \phi(\Bar{m})$ at equation \eqref{eq:weight_modification}.
\end{proof}

\begin{lemma}
    \label{lemma.a.1_Kth}
    The following equation holds, and we use this formula to prove the bias of the OPCB estimator.
    \begin{align*}
        \mE_{p(x)} \left[ \sumM \pi(m|x) \DeltaqfKth \right] = \mE_{p(x) \piBarzeroKth} \left[ \sumM w(x, m) \pi_0(m|x, \phi(\Bar{m})) \DeltaqfKth \right]
    \end{align*}
\end{lemma}
\begin{proof}
    \begin{align*}
         & \mE_{p(x)} \left[ \sumM \pi(m|x) \DeltaqfKth \right] \\
         ={}& \mE_{p(x)} \left[ \sumM w(x, m) \DeltaqfKth  \pi_0(m|x) \right] \quad \because \text{definition of $w(x, m)$} \\
         ={}& \mE_{p(x)} \left[ \sumM w(x, m) \DeltaqfKth  \sum_{\phi(\Bar{m}) \in 2^{\calA}} \piBarzeroKth \pi_0(m|x, \phi(\Bar{m})) \right] \quad \because \text{decomposition of $\pi_0(m|x)$} \\ 
         ={}& \mE_{p(x)} \left[ \sum_{\phi(\Bar{m}) \in 2^{\calA}} \piBarzeroKth \sumM w(x, m) \DeltaqfKth \pi_0(m|x, \phi(\Bar{m})) \right] \quad \because \text{change the order of $\sum$} \\ 
         ={}& \mE_{p(x) \piBarzeroKth} \left[ \sumM w(x, m) \pi_0(m|x, \phi(\Bar{m})) \DeltaqfKth \right] \\ 
    \end{align*}
\end{proof}

\begin{lemma}
    \label{lemma.a.2_Kth}
    The following equation holds, and we use this formula to prove the bias of the OPCB estimator.
    \begin{align*}
        & \mE_{p(x)} \left[ \sumM \pi_0(m|x) \DeltaqfKth w(x, \phi(m)) \right] \\
        ={}& \mE_{p(x) \piBarzeroKth} \left[ \sumMprime w(x, m') \pi_0(m'|x, \phi(\Bar{m})) \sumM \pi_0(m|x, \phi(\Bar{m})) \DeltaqfKth \right]
    \end{align*}
\end{lemma}
\begin{proof}
    \begin{align}
         & \mE_{p(x)} \left[ \sumM \pi_0(m|x) \DeltaqfKth w(x, \phi(m)) \right] \notag \\
         ={}& \mE_{p(x)} \left[ \sumM \pi_0(m|x) \DeltaqfKth \sum_{\phi(\Bar{m}) \in 2^{\calA}} w(x, \phi(\Bar{m})) \ind \{ \phi(\Bar{m}) = \phi(m) \} \right] \notag \quad \because \text{def. of $w(x, \phi(m))$} \\
         ={}& \mE_{p(x)} \left[ \sum_{\phi(\Bar{m}) \in 2^{\calA}} w(x, \phi(\Bar{m})) \sumM \pi_0(m|x) \DeltaqfKth  \ind \{ \phi(\Bar{m}) = \phi(m) \} \right] \notag \quad \because \text{order of $\sum$} \\
         ={}& \mE_{p(x)} \left[ \sum_{\phi(\Bar{m}) \in 2^{\calA}} \piBarzeroKth w(x, \phi(\Bar{m})) \sumM \frac{\pi_0(m|x) \ind \{ \phi(\Bar{m}) = \phi(m) \}}{\piBarzeroKth} \DeltaqfKth \right] \notag  \\
         ={}& \mE_{p(x)} \Bigg[ \sum_{\phi(\Bar{m}) \in 2^{\calA}} \piBarzeroKth \sumMprime w(x, m') \pi_0(m'|x, \phi(\Bar{m})) \notag  \Bigg( \sumM \pi_0(m|x, \phi(\Bar{m})) \DeltaqfKth \Bigg) \Bigg] \notag \quad \because \text{Lemma \ref{lemma.a.3_Kth}}  \\
         ={}& \mE_{p(x) \piBarzeroKth} \left[ \sumMprime w(x, m') \pi_0(m'|x, \phi(\Bar{m})) \sumM \pi_0(m|x, \phi(\Bar{m})) \DeltaqfKth \right] \notag
    \end{align}
\end{proof}

\begin{lemma}
    \label{lemma.a.3_Kth}
    The following equation holds, and we use this formula to prove Lemma \ref{lemma.a.2_Kth}. Note that we we use the following notation $\pi_0(m, \phi(\Bar{m})|x) := \pi_0(m|x) \ind \{ \phi(m) = \phi(\Bar{m}) \}$.
    \begin{align*}
        w(x, \phi(\Bar{m})) = \sumM w(x, m) \pi_0(m|x, \phi(\Bar{m}))
    \end{align*}
\end{lemma}
\begin{proof}
    \begin{align*}
         w(x, \phi(\Bar{m}))  
         &= \frac{\piBarKth}{\piBarzeroKth} \quad \because \text{definition of $w(x, \phi(\Bar{m}))$} \\
         &= \frac{1}{\piBarzeroKth} \sumM \pi(m|x) \ind \{ \phi(m) = \phi(\Bar{m}) \} \quad \because \text{definition of $\piBarKth$} \\
         &= \frac{1}{\piBarzeroKth} \sumM \pi_0(m|x) w(x, m) \ind \{ \phi(m) = \phi(\Bar{m}) \} \quad \because \text{multiply and divide by $\pi_0(m|x)$} \\
         &= \sumM w(x, m) \frac{1}{\piBarzeroKth} \pi_0(m|x) \ind \{ \phi(m) = \phi(\Bar{m}) \} \\
         &= \sumM w(x, m) \frac{\pi_0(m, \phi(\Bar{m})|x)}{\piBarzeroKth} \quad \because \text{definition of $\pi_0(m, \phi(\Bar{m})|x)$} \\
         &= \sumM w(x, m) \pi_0(m|x, \phi(\Bar{m})) \quad \because \text{definition of $\pi_0(m|x, \phi(\Bar{m}))$} \\
    \end{align*}
\end{proof}

\subsection{Proof of Theorem \ref{thm.variance_OPCB_Kth}}
\label{proof_variance_OPCB_Kth}
\begin{proof}
    We prove the variance of the OPCB estimator under Conditions \ref{ass.common_support_Kth} and \ref{ass.local_correctness_Kth} below by using the law of total variance twice.
    \begin{align}
        & n \mV_{\calD} \left[ \Kopca \right] \notag \\
        ={}& n \mV_{p(x_i) \pi_0(m_i|x_i) p(r_i|x_i, m_i) \forall i \in [n]} \left[ \meanN \left\{ \frac{\piiKth}{\piizeroKth} \left( r_i - \fKth(x_i, m_i) \right) + \mE_{\pi(m|x_i)}[\fKth(x_i, m)] \right\} \right] \notag \\
        ={}&  n \frac{1}{n^2} \sum_{i = 1}^n \mV_{p(x_i) \pi_0(m_i|x_i) p(r_i|x_i, m_i) \forall i \in [n]} \left[  \frac{\piiKth}{\piizeroKth} \left( r_i - \fKth(x_i, m_i) \right) + \mE_{\pi(m|x_i)}[\fKth(x_i, m)]  \right] \notag \\
        ={}& \mV_{p(x) \pi_0(m|x) p(r|x, m)} \left[ w(x, \phi(m)) \left( r - \fKth(x, m) \right) + \mE_{\pi(m'|x)}[\fKth(x, m')] \right] \notag  \quad \because \text{i.i.d. assumption} \\
        ={}& \mE_{p(x) \pi_0(m|x)} \left[ \mV_{p(r|x, m)} \left[ w(x, \phi(m)) \left( r - \fKth(x, m) \right) + \mE_{\pi(m'|x)}[\fKth(x, m')] \right] \right] \notag \\
        & + \mV_{p(x) \pi_0(m|x)} \left[ \mE_{p(r|x, m)} \left[ w(x, \phi(m)) \left( r - \fKth(x, m) \right) + \mE_{\pi(m'|x)}[\fKth(x, m')] \right] \right] \notag  \quad \because \text{total variance} \\
        ={}& \mE_{p(x) \pi_0(m|x)} \left[ w(x, \phi(m))^2 \sigma^2(x, m) \right] \notag \\
        & + \mV_{p(x) \pi_0(m|x)} \left[  w(x, \phi(m)) \left( q(x, m) - \fKth(x, m) \right) + \mE_{\pi(m'|x)}[\fKth(x, m')] \right] \notag  \quad \because \text{definition of $\sigma^2(x, m)$} \\
        ={}& \mE_{p(x) \pi_0(m|x)} \left[ w(x, \phi(m))^2 \sigma^2(x, m) \right] \notag \\
        & + \mV_{p(x) \pi_0(m|x)} \left[  w(x, \phi(m)) \DeltaqfKth + \mE_{\pi(m'|x)}[\fKth(x, m')] \right] \notag \quad \because \text{definition of $\Delta_{q, \fKth(x, m)}$} \\
        ={}& \mE_{p(x) \pi_0(m|x)} \left[ w(x, \phi(m))^2 \sigma^2(x, m) \right] \notag \\
        & + \mE_{p(x)} \left[ \mV_{\pi_0(m|x)} \left[ w(x, \phi(m)) \DeltaqfKth + \mE_{\pi(m'|x)}[\fKth(x, m')] \right] \right] \notag \\
        & + \mV_{p(x)} \left[ \mE_{\pi_0(m|x)} \left[ w(x, \phi(m)) \DeltaqfKth + \mE_{\pi(m'|x)}[\fKth(x, m')] \right] \right] \notag  \quad \because \text{total variance formula} \\
        ={}& \mE_{p(x) \pi_0(m|x)} \left[ w(x, \phi(m))^2 \sigma^2(x, m) \right] \notag \\
        & + \mE_{p(x)} \left[ \mV_{\pi_0(m|x)} \left[ w(x, \phi(m)) \DeltaqfKth \right] \right] \notag \\
        & + \mV_{p(x)} \left[ \mE_{\pi_0(m|x)} \left[ w(x, \phi(m)) \Bar{\Delta}(x, \phi(m)) + \mE_{\pi(m'|x)}[\fKth(x, m')] \right] \right] \notag \quad \because \text{conditional piecewise correctness} \\
        ={}& \mE_{p(x) \pi_0(m|x)} \left[ w(x, \phi(m))^2 \sigma^2(x, m) \right] \notag \\
        & + \mE_{p(x)} \left[ \mV_{\pi_0(m|x)} \left[ w(x, \phi(m)) \DeltaqfKth \right] \right] \notag \\
        & + \mV_{p(x)} \left[ \mE_{\pi(m|x)} \left[ q(x, m) \right] \right] \notag \quad \because \text{Proof of Theorem \ref{thm.unbiasedness_OPCB_Kth}}
    \end{align}
\end{proof}

\subsection{Proof of Theorem \ref{thm-unbiasedness-OPCB-PG}}
\label{proof-unbiasedness-OPCB-PG}
\begin{proof}
    We can show the unbiasedness of the gradient of the OPCB estimator under Conditions \ref{ass.full-support-main-actions} and \ref{ass.local_correctness_Kth} below.
    \begin{align*}
        & \mE_{\calD} \left[ \nabla_{\zeta} \Kopcazeta \right] \\
        ={}& \mE_{p(x_i) \pi_0(m_i|x_i) p(r_i|x_i, m_i) \forall i \in [n]} \left[ \meanN \left\{ \frac{\piizetaphi}{\piizeroKth} \left( r_i - \fKth(x_i, m_i) \right) s_{\zeta}(x_i, \phi(m_i)) + \mE_{\pi_{\zeta}(m|x_i)}[\fKth(x_i, m) s_{\zeta}(x_i, m)] \right\} \right] \\
        ={}& \meanN \mE_{p(x_i) \pi_0(m_i|x_i) p(r_i|x_i, m_i)} \left[  \frac{\piizetaphi}{\piizeroKth} \left( r_i - \fKth(x_i, m_i) \right) s_{\zeta}(x_i, \phi(m_i)) + \mE_{\pi_{\zeta}(m|x_i)}[\fKth(x_i, m) s_{\zeta}(x_i, m)]  \right] \quad \because \text{linearity of $\mE$}  \\
        ={}&  \mE_{p(x) \pi_0(m|x) p(r|x, m)} \left[ \frac{\pizetaphi}{\pizeroKth} \left( r - \fKth(x, m) \right) s_{\zeta}(x, \phi(m)) + \mE_{\pi_{\zeta}(m'|x)}[\fKth(x, m') s_{\zeta}(x, m')] \right] \quad \because \text{i.i.d. assumption}  \\
        ={}&  \mE_{p(x) \pi_0(m|x)} \left[ \frac{\pizetaphi}{\pizeroKth} \left( q(x, m) - \fKth(x, m) \right) s_{\zeta}(x, \phi(m)) + \mE_{\pi_{\zeta}(m'|x)}[\fKth(x, m') s_{\zeta}(x, m')]\right] \quad \because \text{definition of $q(x, m)$} \\
        ={}&  \mE_{p(x)} \left[ \sumM \pi_0(m|x) \frac{\pizetaphi}{\pizeroKth} \left( q(x, m) - \fKth(x, m) \right) s_{\zeta}(x, \phi(m)) + \mE_{\pi_{\zeta}(m'|x)}[\fKth(x, m') s_{\zeta}(x, m')] \right] \\
        ={}&  \mE_{p(x)} \left[ \sumM \pi_0(m|x) \frac{\pizetaphi)}{\pizeroKth} \Bar{\Delta}(x, \phi(m)) s_{\zeta}(x, \phi(m)) + \mE_{\pi_{\zeta}(m'|x)}[\fKth(x, m') s_{\zeta}(x, m')] \right] \quad \because \text{conditional piecewise correctness}  \\
        ={}&  \mE_{p(x)} \left[ \sumM \pi_0(m|x) \sum_{\phi(m') \in 2^{\calA}} \ind \{ \phi(m') = \phi(m) \} \frac{\piprimezetaphi}{\piprimezeroKth} \Bar{\Delta}(x, \phi(m')) s_{\zeta}(x, \phi(m')) + \mE_{\pi_{\zeta}(m'|x)}[\fKth(x, m') s_{\zeta}(x, m')] \right] \\
        ={}&  \mE_{p(x)} \left[ \sum_{\phi(m') \in 2^{\calA}} \frac{\piprimezetaphi}{\piprimezeroKth} \Bar{\Delta}(x, \phi(m')) s_{\zeta}(x, \phi(m')) \sumM \pi_0(m|x) \ind \{ \phi(m') = \phi(m) \} + \mE_{\pi_{\zeta}(m'|x)}[\fKth(x, m') s_{\zeta}(x, m')] \right] \\
        ={}&  \mE_{p(x)} \left[ \sum_{\phi(m') \in 2^{\calA}} \frac{\piprimezetaphi}{\piprimezeroKth} \Bar{\Delta}(x, \phi(m')) s_{\zeta}(x, \phi(m')) \piprimezeroKth + \mE_{\pi_{\zeta}(m'|x)}[\fKth(x, m') s_{\zeta}(x, m')] \right] \quad \because \text{def. $\piprimezeroKth$} \\
        ={}&  \mE_{p(x)} \left[ \sum_{\phi(m') \in 2^{\calA}} \piprimezetaphi \Bar{\Delta}(x, \phi(m')) s_{\zeta}(x, \phi(m')) + \mE_{\pi_{\zeta}(m'|x)}[\fKth(x, m') s_{\zeta}(x, m')] \right] \quad \because \text{cancel out $\piprimezeroKth)$} \\
        ={}&  \mE_{p(x)} \left[ \sum_{\phi(m') \in 2^{\calA}}  \Bar{\Delta}(x, \phi(m')) \nabla_{\zeta} \pi_{\zeta}(\phi(m')|x) + \mE_{\pi_{\zeta}(m'|x)}[\fKth(x, m') s_{\zeta}(x, m')] \right]  \\
        ={}&  \mE_{p(x)} \left[ \sum_{\phi(m') \in 2^{\calA}}  \Bar{\Delta}(x, \phi(m')) \nabla_{\zeta} \sumM \pi_{\zeta}(m|x) \ind\{ \phi(m) = \phi(m') \} + \mE_{\pi_{\zeta}(m'|x)}[\fKth(x, m') s_{\zeta}(x, m')] \right]  \\
        ={}&  \mE_{p(x)} \left[ \sumM \nabla_{\zeta} \pi_{\zeta}(m|x) \sum_{\phi(m') \in 2^{\calA}}  \Bar{\Delta}(x, \phi(m'))  \ind\{ \phi(m) = \phi(m') \} + \mE_{\pi_{\zeta}(m'|x)}[\fKth(x, m') s_{\zeta}(x, m')] \right]  \\
        ={}&  \mE_{p(x)} \left[ \sumM \nabla_{\zeta} \pi_{\zeta}(m|x) \Delta_{q, \hat{f}}(x, m) + \mE_{\pi_{\zeta}(m'|x)}[\fKth(x, m') s_{\zeta}(x, m')] \right]  \quad \because \text{conditional piecewise correctness} \\
        ={}& \mE_{p(x)} \left[ \sumM \pi_{\zeta}(m|x) s_{\zeta}(x, m)  \Delta_{q, \hat{f}}(x, m)  + \mE_{\pi_{\zeta}(m'|x)}[\fKth(x, m') s_{\zeta}(x, m')] \right] \\
        ={}& \mE_{p(x) \pi_{\zeta}(m|x)} \left[ q(x, m) s_{\zeta}(x, m) \right] \\
        ={}& \nabla_{\zeta} V(\pi_{\zeta}).
    \end{align*}
\end{proof}

\subsection{Proof of Theorem \ref{thm-bias-OPCB-PG}}
\label{proof-bias-OPCB-PG}
\begin{proof}
    We derive the bias of the gradient of the OPCB estimator under Condition \ref{ass.full-support-main-actions}.
    \begin{align}
        & \bias \left( \nabla_{\zeta} \Kopcazeta \right) \notag \\
        ={}& \mE_{\calD} \left[ \nabla_{\zeta} \Kopcazeta \right] - \nabla_{\zeta} V(\pi_{\zeta}) \notag \\
        ={}& \nabla_{\zeta}  \mE_{\calD} \left[ \Kopcazeta \right] - \nabla_{\zeta} V(\pi_{\zeta}) \notag \\
        ={}& \nabla_{\zeta} \left( \mE_{\calD} \left[ \Kopcazeta \right] - V(\pi_{\zeta}) \right) \notag \\
        ={}& \nabla_{\zeta} \bias \left(  \Kopcazeta \right)\notag \\
        ={}& \nabla_{\zeta} \mE_{p(x) \pibarzetaphi} \Bigg[ \sum_{\substack{m < m':\\ \phi(m) = \phi(m') = \phi(\Bar{m})}} \pi_0(m|x, \phi(\Bar{m})) \pi_0(m'|x, \phi(\Bar{m})) \notag \\
        & \times \left( \Delta_q(x, m, m') - \Delta_{\fKth}(x, m, m') \right) \left( \frac{\pi_{\zeta}(m'|x, \phi(\Bar{m}))}{\pi_0(m'|x, \phi(\Bar{m}))} - \frac{\pi_{\zeta}(m|x, \phi(\Bar{m}))}{\pi_0(m|x, \phi(\Bar{m}))} \right) \Bigg] \quad \because \text{Theorem \ref{thm.bias_OPCB_Kth}} \notag \\
        ={}& \mE_{p(x)} \Bigg[ \nabla_{\zeta} \sumMBarKth \pibarzetaphi \sum_{\substack{m < m':\\ \phi(m) = \phi(m') = \phi(\Bar{m})}} \pi_0(m|x, \phi(\Bar{m})) \pi_0(m'|x, \phi(\Bar{m})) \notag \\
        & \times \left( \Delta_q(x, m, m') - \Delta_{\fKth}(x, m, m') \right) \left( \frac{\pi_{\zeta}(m'|x, \phi(\Bar{m}))}{\pi_0(m'|x, \phi(\Bar{m}))} - \frac{\pi_{\zeta}(m|x, \phi(\Bar{m}))}{\pi_0(m|x, \phi(\Bar{m}))} \right) \Bigg]  \notag \\
        ={}& \mE_{p(x)} \Bigg[ \sumMBarKth \sum_{\substack{m < m':\\ \phi(m) = \phi(m') = \phi(\Bar{m})}} \pi_0(m|x, \phi(\Bar{m})) \pi_0(m'|x, \phi(\Bar{m})) \notag \\
        & \times \left( \Delta_q(x, m, m') - \Delta_{\fKth}(x, m, m') \right) \nabla_{\zeta} \pibarzetaphi \left( \frac{\pi_{\zeta}(m'|x, \phi(\Bar{m}))}{\pi_0(m'|x, \phi(\Bar{m}))} - \frac{\pi_{\zeta}(m|x, \phi(\Bar{m}))}{\pi_0(m|x, \phi(\Bar{m}))} \right) \Bigg]  \notag \\
        ={}& \mE_{p(x)} \Bigg[ \sumMBarKth \sum_{\substack{m < m':\\ \phi(m) = \phi(m') = \phi(\Bar{m})}} \pi_0(m|x, \phi(\Bar{m})) \pi_0(m'|x, \phi(\Bar{m})) \notag \\
        & \times \left( \Delta_q(x, m, m') - \Delta_{\fKth}(x, m, m') \right)  \pibarzetaphi \left( \frac{\pi_{\zeta}(m'|x, \phi(\Bar{m}))}{\pi_0(m'|x, \phi(\Bar{m}))} \nabla_{\zeta} \log \pi_{\zeta}(m', \phi(\Bar{m})|x) - \frac{\pi_{\zeta}(m|x, \phi(\Bar{m}))}{\pi_0(m|x, \phi(\Bar{m}))} \nabla_{\zeta} \log \pi_{\zeta}(m, \phi(\Bar{m})|x) \right) \Bigg]  \notag \\
        ={}& \mE_{p(x) \pibarzetaphi} \Bigg[ \sum_{\substack{m < m':\\ \phi(m) = \phi(m') = \phi(\Bar{m})}} \pi_0(m|x, \phi(\Bar{m})) \pi_0(m'|x, \phi(\Bar{m})) \notag \\
        & \times \left( \Delta_q(x, m, m') - \Delta_{\fKth}(x, m, m') \right) \left( \frac{\pi_{\zeta}(m'|x, \phi(\Bar{m}))}{\pi_0(m'|x, \phi(\Bar{m}))} \nabla_{\zeta} \log \pi_{\zeta}(m', \phi(\Bar{m})|x) - \frac{\pi_{\zeta}(m|x, \phi(\Bar{m}))}{\pi_0(m|x, \phi(\Bar{m}))} \nabla_{\zeta} \log \pi_{\zeta}(m, \phi(\Bar{m})|x) \right) \Bigg]  \notag 
    \end{align}
\end{proof}

\subsection{Proof of Theorem \ref{thm-variance-OPCB-PG}}
\label{proof-variance-OPCB-PG}
\begin{proof}
    We prove the variance of the gradient of the OPCB estimator under Conditions \ref{ass.full-support-main-actions} and \ref{ass.local_correctness_Kth} below by using the law of total variance twice.
    \begin{align}
        & n \mV_{\calD} \left[ \nabla_{\zeta} \Kopcazeta \right] \notag \\
        ={}& n \mV_{p(x_i) \pi_0(m_i|x_i) p(r_i|x_i, m_i) \forall i \in [n]} \left[ \meanN \left\{ \frac{\piizetaphi}{\piizeroKth} \left( r_i - \fKth(x_i, m_i) \right) s_{\zeta}(x_i, \phi(m_i)) + \mE_{\pi_{\zeta}(m|x_i)}[\fKth(x_i, m) s_{\zeta}(x_i, m)] \right\} \right] \notag \\
        ={}&  n \frac{1}{n^2} \sum_{i = 1}^n \mV_{p(x_i) \pi_0(m_i|x_i) p(r_i|x_i, m_i) \forall i \in [n]} \left[  \frac{\piizetaphi}{\piizeroKth} \left( r_i - \fKth(x_i, m_i) \right) s_{\zeta}(x_i, \phi(m_i)) + \mE_{\pi_{\zeta}(m|x_i)}[\fKth(x_i, m) s_{\zeta}(x_i, m)]  \right] \notag \\
        ={}& \mV_{p(x) \pi_0(m|x) p(r|x, m)} \left[ \frac{\pizetaphi}{\pizeroKth} \left( r - \fKth(x, m) \right) s_{\zeta}(x, \phi(m)) + \mE_{\pi_{\zeta}(m'|x)}[\fKth(x, m') s_{\zeta}(x, m')] \right] \notag  \quad \because \text{i.i.d. assumption} \\
        ={}& \mE_{p(x) \pi_0(m|x)} \left[ \mV_{p(r|x, m)} \left[ \frac{\pizetaphi}{\pizeroKth} \left( r - \fKth(x, m) \right) s_{\zeta}(x, \phi(m)) + \mE_{\pi_{\zeta}(m'|x)}[\fKth(x, m') s_{\zeta}(x, m')] \right] \right] \notag \\
        & + \mV_{p(x) \pi_0(m|x)} \left[ \mE_{p(r|x, m)} \left[ \frac{\pizetaphi}{\pizeroKth} \left( r - \fKth(x, m) \right) s_{\zeta}(x, \phi(m)) + \mE_{\pi_{\zeta}(m'|x)}[\fKth(x, m') s_{\zeta}(x, m')] \right] \right] \notag  \quad \because \text{total variance} \\
        ={}& \mE_{p(x) \pi_0(m|x)} \left[ \left(\frac{\pizetaphi}{\pizeroKth}\right)^2 \chi_{\zeta}(x, \phi(m)) \sigma^2(x, m) \right] \notag \\
        & + \mV_{p(x) \pi_0(m|x)} \left[  \frac{\pizetaphi}{\pizeroKth} \left( q(x, m) - \fKth(x, m) \right) s_{\zeta}(x, \phi(m)) + \mE_{\pi_{\zeta}(m'|x)}[\fKth(x, m') s_{\zeta}(x, m')] \right] \notag  \quad \because \text{definition of $\sigma^2(x, m)$ and $\chi_{\zeta}(x, \phi(m))$} \\
        ={}& \mE_{p(x) \pi_0(m|x)} \left[ \left(\frac{\pizetaphi}{\pizeroKth}\right)^2 \chi_{\zeta}(x, \phi(m)) \sigma^2(x, m) \right] \notag \\
        & + \mV_{p(x) \pi_0(m|x)} \left[  \frac{\pizetaphi}{\pizeroKth} \DeltaqfKth s_{\zeta}(x, \phi(m)) + \mE_{\pi_{\zeta}(m'|x)}[\fKth(x, m') s_{\zeta}(x, m')] \right] \notag \quad \because \text{definition of $\Delta_{q, \fKth(x, m)}$} \\
        ={}& \mE_{p(x) \pi_0(m|x)} \left[ \left(\frac{\pizetaphi}{\pizeroKth}\right)^2 \chi_{\zeta}(x, \phi(m)) \sigma^2(x, m) \right] \notag \\
        & + \mE_{p(x)} \left[ \mV_{\pi_0(m|x)} \left[ \frac{\pizetaphi}{\pizeroKth} \DeltaqfKth s_{\zeta}(x, \phi(m)) + \mE_{\pi_{\zeta}(m'|x)}[\fKth(x, m') s_{\zeta}(x, m')] \right] \right] \notag \\
        & + \mV_{p(x)} \left[ \mE_{\pi_0(m|x)} \left[ \frac{\pizetaphi}{\pizeroKth} \DeltaqfKth s_{\zeta}(x, \phi(m)) + \mE_{\pi_{\zeta}(m'|x)}[\fKth(x, m') s_{\zeta}(x, m')] \right] \right] \notag  \quad \because \text{total variance formula} \\
        ={}& \mE_{p(x) \pi_0(m|x)} \left[ \left(\frac{\pizetaphi}{\pizeroKth}\right)^2 \chi_{\zeta}(x, \phi(m)) \sigma^2(x, m) \right] \notag \\
        & + \mE_{p(x)} \left[ \mV_{\pi_0(m|x)} \left[ \frac{\pizetaphi}{\pizeroKth} \DeltaqfKth s_{\zeta}(x, \phi(m)) \right] \right] \notag \\
        & + \mV_{p(x)} \left[ \mE_{\pi_0(m|x)} \left[ \frac{\pizetaphi}{\pizeroKth} \Bar{\Delta}(x, \phi(m)) s_{\zeta}(x, \phi(m)) + \mE_{\pi_{\zeta}(m'|x)}[\fKth(x, m') s_{\zeta}(x, m')] \right] \right] \notag \quad \because \text{conditional piecewise correctness} \\
        ={}& \mE_{p(x) \pi_0(m|x)} \left[ \left(\frac{\pizetaphi}{\pizeroKth}\right)^2 \chi_{\zeta}(x, \phi(m)) \sigma^2(x, m) \right] \notag \\
        & + \mE_{p(x)} \left[ \mV_{\pi_0(m|x)} \left[ \frac{\pizetaphi}{\pizeroKth} \DeltaqfKth s_{\zeta}(x, \phi(m)) \right] \right] \notag \\
        & + \mV_{p(x)} \left[ \mE_{\pi(m|x)} \left[ q(x, m) s_{\zeta}(x, m) \right] \right] \notag \quad \because \text{Proof of Theorem \ref{thm-unbiasedness-OPCB-PG}}
    \end{align}
\end{proof}

\section{Detailed Experiment Settings and Results}
This section describes the detailed experiment settings and reports additional results.

\subsection{Synthetic Experiments \label{sec:appendix synthetic}}
\textbf{Detailed Setup.} \quad We describe synthetic experiment settings in detail. First, we describe how we define the expected reward function. In the synthetic experiments, we define the expected reward function as follows.
\begin{align}
    q(x, m) = \lambda g(x, \phi_{true}(m)) + (1 - \lambda) h(x, m), \notag
\end{align}
where we define $g(\cdot,\cdot)$ and $h(\cdot,\cdot)$ in the following way. We use \textbf{obp.dataset.linear\_reward\_function} from OpenBanditPipeline \cite{saito2020open} as $g(\cdot,\cdot)$. We define $h(\cdot,\cdot)$ as
\begin{align}
    h(x,m) = \theta^T_x x \theta^T_m m  + \epsilon_{x,m}, \notag
\end{align}
where $\theta_x$ and $\theta_m$ are sampled from a uniform distribution with range $[-1.5,1.5]$. $\epsilon_{x,m}$ is sampled from a uniform distribution with range $[-2.5,2.5]$.
\\\\
Second, we describe the definitions of the logging policy $\pi_0$ and target policy $\pi$ when we vary $\lambda$ in Figure \ref{fig:lambda}. When we vary the ratio of the main elements, the expected reward $q(x,m)$ changes. This affects the difficulty of OPE because the logging policy defined by Eq. \eqref{eq:logging} changes. Therefore, we should use definitions of target and logging policies not depending on $\lambda$. We use \textbf{obp.dataset.linear\_behavior\_policy} from OpenBanditPipeline \cite{saito2020open} as the logging policy $\pi_0$. In contrast, we define the target policy $\pi$ by applying the softmax function to a normal distribution with mean $q(x,m)$ and standard deviation $\sigma = 3.0$ like Eq. \eqref{eq:logging} where we use $\beta = 3.0$. By using these definitions, it is possible to conduct OPE experiments whose difficulties are consistent when we vary the value of $\lambda$. 
\\\\
Third, we describe how to optimize the function $\phi$ of OPCB. For instance, it has 256 candidates when $|\calA|=8$ ($|\calS|=256$). It is difficult to explore all of them, so we use random search and try 30 candidate functions $\phi$. We perform Eq.~\eqref{eq:tuning} using the following bias estimation.
\begin{align}
    \label{eq:bias_estimation}
    \widehat{\bias} \left(\Kopca\right)^2 = \bias \left(\Kopca\right)^2 + \delta_{\phi},
\end{align}
where $\delta_{\phi}$ is sampled from a normal distribution with mean $0$ and standard deviation $\sigma$. We use $\sigma = 2.5$ as default. In one of the additional synthetic experiments, we vary $\sigma$ to investigate how $\sigma$ affects the performance of OPCB.
\\\\
\textbf{Additional Results.} \quad In Figure \ref{fig:reward_std} to \ref{fig:bias_noise}, we report additional results on synthetic experiments in order to demonstrate that OPCB works better than the baselines in various situations. We vary levels of noise on the rewards and target policies. Additionally, we vary levels of noise on the bias estimation in Eq. \eqref{eq:bias_estimation}. We set $n = 2000$ and $|\calA|=8$. Note that the shaded regions in the MSE plots represent the 95\% confidence intervals estimated with bootstrap.
\\\\
\textbf{How does OPCB perform with varying levels of noise on the rewards, target policies, and levels of noise on the bias estimation?} \quad
First, Figure \ref{fig:reward_std} shows comparisons of the estimators with varying reward noise levels $\sigma \in \{1.0,2.0,3.0,4.0,5.0\}$. Note that the reward is sampled from a normal distribution with mean $q(x,m)$ and standard deviation $\sigma$. We observe that the variance of IPS, DR, and OPCB (ours) gradually increases as the reward noises increase, and OPCB (ours) significantly outperforms IPS and DR with larger reward noises because of a greater reduction of variance. Moreover, OPCB (ours) has a lower bias than DM for various noise levels. OPCB (ours) performs better than the baselines with larger noise levels.
\\\\
Second, Figure \ref{fig:target_policy} shows how target policies affect the performance of estimators. We vary target policies ($\epsilon \in \{0.0,0.2,0.4,0.6,0.8,1.0\}$ in Eq. \eqref{eq:evaluation}). A large value of $\epsilon$ makes the target policy close to a random uniform distribution. We observe that all estimators become worse with a smaller value of $\epsilon$ since it makes OPE more challenging. IPS and DR suffer from high variance, and DR deteriorates due to high bias. However, OPCB (ours) achieves lower MSE with various target policies. Especially when $\epsilon$ is small, OPCB (ours) significantly improves the MSE compared to the baselines. The results show that OPCB can perform better particularly in more practical and challenging situations.
\\\\
Finally, Figure \ref{fig:bias_noise} shows how the estimation errors of the bias term affect OPCB's performance. We vary noise levels $\sigma \in \{ 0.0,1.0,2.0,3.0,4.0,5.0 \}$ in Eq. \eqref{eq:bias_estimation}. Note that the bias estimation noises do not affect the estimators other than OPCB (ours). We see that the performance of OPCB (ours) deteriorates as the value of $\sigma$ increases. This suggests that the estimation errors of the bias term can control the precision of estimating the MSE. We set $\sigma = 2.5$ to ensure fair comparisons in the main text.

\subsection{OPL Experiments} \label{app:gradients}
In this section, we describe the baselines and detailed settings in OPL experiments. To conduct synthetic experiments for OPL, we create the logged data $\calD = \{(x_i, m_i, r_i)\}_{i = 1}^n$ similarly to OPE experiments. We use a neural network with 3 hidden layers to parameterize the policy $\pi_\zeta$. Below, we define the baseline methods. 
\\\\
\textbf{Reg-based.} The regression-based approach learns the expected reward function $q(x,m)$. Then, it transforms the estimated reward function into a policy. We use a neural network with 3 hidden layers to obtain the estimated reward function $\hat{q}_{\zeta}(x,m)$. And then we define the policy $\pi_\zeta$ as follows.
\begin{align}
    \pi_\zeta(x,m) = \frac{\exp(\beta \cdot \hat{q}_{\zeta}(x,m))}{\sum_{m' \in \calM} \exp(\beta \cdot \hat{q}_{\zeta}(x,m'))},
\end{align}
where we use $\beta = 10$. The regression-based approach can avoid variance issues. However, it suffers from severe bias due to the difficulty of accurately learning the expected rewards.
\\ \\
\textbf{IPS-PG.} The IPS policy gradient estimator is defined as
\begin{align}
    \label{ipspg}
    \nabla_\zeta \hat{V}_{\mathrm{IPS}} (\pi_\zeta;\calD)
    := \meanN \frac{\pi_\zeta(m_i|x_i)}{\pi_0(m_i|x_i)} r_i \nabla_\zeta \log{\pi_\zeta(m_i|x_i)}
    = \meanN w(x_i,m_i) r_i s_\zeta(x_i,m_i),
\end{align}
where $w(x,m) = \pi_\zeta(m|x) / \pi_0(m|x)$ is the (vanilla) importance weight and $s_\zeta(x,m) = \nabla_\zeta \log{\pi_\zeta(m|x)}$ is the policy score function. IPS-PG is unbiased under the full support condition ($\pi_0(m|x) > 0$ for all $m \in \calM, x \in \calX$). However, IPS-PG suffers from severe variance in the presence of many actions. Moreover, in large action spaces, the full support condition is likely to be violated, leading to substantial bias~\citep{sachdeva2020off}.
\\ \\
\textbf{DR-PG.} The DR policy gradient estimator is defined as
\begin{align}
    \label{drpg}
    \nabla_\zeta \hat{V}_{\mathrm{DR}} (\pi_\zeta;\calD, \hat{q})
    := \meanN w(x_i,m_i) \left( r_i - \hat{q}(x_i,m_i) \right) s_\zeta(x_i,m_i) + \mE_{\pi_\zeta(m|x_i)} \left[ \hat{q}(x_i,m) s_\zeta(x_i,m) \right]
\end{align}
DR-PG is unbiased under the full support condition and DR-PG can reduce variance compared to IPS-PG. However, it still struggles with high variance in large action spaces because it uses the vanilla importance weight like IPS-PG.
\\ \\
\textbf{OPCB-PG.} The OPCB policy gradient estimator is defined as
\begin{align}
    \label{opcbpg}
    \nabla_\zeta \hat{V}_{\mathrm{OPCB}} (\pi_\zeta;\calD, \phi)
    := \meanN \frac{\piizetaphi}{\piizeroKth} \left( r_i - \hat{f}(x_i,m_i) \right) \nabla_\zeta \log{\pi_\zeta(\phi(m_i)|x_i)} + \mE_{\pi_\zeta(m|x_i)} \left[ \hat{f}(x_i,m) \nabla_\zeta \log{\pi_\zeta(m|x_i)} \right]
\end{align}
OPCB-PG is unbiased under the full support condition and the conditional pairwise correctness. OPCB-PG can reduce variance by using the marginalized importance weight regarding the main actions $\phi(m)$.\\

\subsection{Real-World Experiments}
In this section, we describe the detailed real-world experiment settings in the main text and conduct an additional real-world experiment with the PenDigits dataset \cite{alpaydin98pen}.
\\ \\
\textbf{Detailed Setup.} \quad We describe the real-world experiment settings on KuaiRec~\citep{gao2022kuairec} in detail. KuaiRec has user features and user-item interactions which are almost fully observed with nearly 100\% density for the subset of its users and items. We consider user features as contexts $x$, where we reduce feature  dimensions based on PCA implemented in scikit-learn \cite{pedregosa2011scikit}. Then, we consider the user-item interactions as the base reward function $\Tilde{q}(x,a)$ and define the expected reward function as follows.
\begin{align}
    q(x, m) = \left( \prod_{l=1}^L \tilde{q}(x, m_l) \right)^{(\sum_{l=1}^L \mathbb{I} \{ m_l = a_l \})^{-1}} \notag
\end{align}
where $\tilde{q}(x, m_l) = \tilde{q}(x, a_l)$ when $m_l = a_l$ and $\tilde{q}(x, m_l) = 1$ when $m_l = \emptyset$. We then sample the reward $r$ from a normal distribution whose mean is $q(x,m)$ and standard deviation $\sigma$ is 3.0. Note that we randomly select 10 items from the user-item interactions, which means that we set $|\calA|=10$ in this experiment.

We define the logging policy and the target policy with an estimated reward function $\hat{q}(x,m)$ as follows.
\begin{align}
    \pi_0(m\,|\,x) 
    &= \frac{\exp(\beta \cdot \hat{q}(x,m))}{\sum_{m' \in \calM} \exp(\beta \cdot \hat{q}(x,m'))}, \label{pi_0_real}\\
    \pi(m\,|\,x) 
    &= (1-\epsilon) \ind \{m = \underset{m' \in \calM} {\operatorname{argmax}} \ q(x,m')\} + \frac{\epsilon}{|\calM|},
\end{align}
where we use $\beta = -0.3$ and $\epsilon = 0.1$. We obtain $\hat{q}(x,m)$ by ridge regression. Note that we separate the dataset for obtaining the estimated reward function and conducting OPE experiments.

To summarize, we first define the expected reward $q(x,m)$. We then obtain a combinatorial action $m$ sampled by $\pi_0$ in Eq. \eqref{pi_0_real}. After that, we sample a reward $r$ from a normal distribution with mean $q(x, m)$ and standard deviation $\sigma = 3.0$. Repeating above procedure $n$ times generates the logged data $\calD = \{x_i, m_i, r_i\}_{i = 1}^n$.
\\\\
\textbf{Additional Experiment Setup for PenDigits.} \quad To conduct an additional experiment on PenDigits, we transform it to contextual combinatorial bandit feedback data. PenDigits contains the bitmap features of the image as the context information and each image is attached with a true label of the digits. To simulate the CCB problem on this dataset, we define the expected reward function as follows.
\begin{align}
    q(x,m) = \Bigg\{
\begin{array}{l} 
    \frac{1}{\sum_{l=1}^L \mathbb{I} \{ m_l = a_l \}} \left( 1 - \eta_{x,m} \right)  \quad \text{if $a \in m$ and $a$ is the \textit{positive} label}\\
    \frac{1}{\sum_{l=1}^L \mathbb{I} \{ m_l = a_l \}} \eta_{x,m} \qquad\qquad \text{otherwise}
    \end{array}, \notag
\end{align}
where $\eta_{x,m}$ denotes a noise parameter sampled from a uniform distribution with range $[0,0.5]$.
The expected reward function $q(x,m)$ becomes larger when a positive label action is selected without including other actions. We then sample the reward $r$ from a normal distribution whose mean is $q(x,m)$ and standard deviation $\sigma$ is 3.0. We define the logging policy by applying the softmax function to an estimated reward function $\hat{q}(x,m)$ like Eq. \eqref{eq:logging} where $\beta = -0.3$ and we obtain $\hat{q}(x,m)$ by ridge regression. We also define a target policy $\pi(m|x)$ like Eq. \eqref{eq:logging} with $\epsilon = 0.1$.

Similarly to the real-world experiment in the main text, we vary the estimation errors of the bias term ($\sigma=1.0, 3.0, 5.0$) to sample $\delta_{\phi}$ in Eq.~\eqref{eq:bias_error}. A larger value of $\sigma$ indicates a lower accuracy in estimating MSE. 
\\\\
\textbf{Results.} \quad 
In Figure \ref{fig:pen}, we show comparisons of the estimators' performance on PenDigits with varying logged data sizes ($n$). The number of actions $|\calA|$ is 10. Note that the shaded regions in the MSE plots represent the 95\% confidence intervals estimated with bootstrap.
\\\\
\textbf{How does OPCB perform with varying logged data sizes on PenDigits?} \quad 
Figure \ref{fig:pen} compares the MSE, squared bias, and variance of the estimators with varying logged data sizes $n \in \{ 500, 1000, 2000, 4000, 8000 \}$ on PenDigits. We observe that OPCB outperforms the baseline methods across various logged data sizes by effectively reducing the bias and variance. Remarkably, OPCB reduces MSE even with the largest noise on its bias estimation similarly to Figure~\ref{fig:kuairec}. These results suggest that OPCB is indeed valuable in practice.

\subsection{Slate OPE Experiments} \label{app:slate_experiment}
As we explain in Appendix \ref{app:slate_ope}, OPCB is applicable to the slate settings. We conduct some experiments in order to verify the effectiveness of OPCB in the slate settings .\\\\
\textbf{Formulation for Slate OPE.} \quad We formulate OPE under a slate contextual bandit setting. We use $\bx\in \calX\subseteq{\mathbb{R}^{d_x}}$ to denote a $d_x$-dimensional context vector drawn i.i.d. from an unknown distribution $p(\bx)$, and $\bs \in \calS \coloneqq \prod_{l=1}^{L} \calA_l$ to denote a slate action. A slate action consists of several sub-actions, i.e., $\bs = (a_1, \cdots, a_L)$. Each sub-action $a_l$ is chosen from $\calA_l$. Note that we define $\calA_l$ as $\calA_l \coloneqq (\emptyset, a_{l,1}, \cdots , a_{l,|\calA_l|})$. In this section, we consider a \textit{factorizable} policy $\pi : \calX \to \Delta(\calS)$ for simplicity. Given a context vector, it chooses a sub-action at each position independently, where $\pi(\bs|\bx) = \prod_{l=1}^L \pi(a_l|\bx)$ is the probability of choosing a specific slate action $\bs$.
\\\\
\textbf{Setup for Slate Experiments.} \quad We conduct slate OPE experiments similarly to synthetic experiments in the main text. We define the expected reward as follows.
\begin{align}
    q(\bx,\bs) = \frac{1}{|\bs|} \sum_{l=1}^{L} q_l(\bx,a_l),
\end{align}
where we use $|\bs|$ to denote the number of actions in a slate. Since there may be slots where nothing is chosen, the number of actions in a slate may not necessarily correspond to the slate size $L$. We then sample the reward $r$ from a normal distribution whose mean is $q(\bx,\bs)$ and standard deviation $\sigma$ is 3.0.
We define the logging policy and the target policy as follows.
\begin{align}
    \pi_0(\bs|\bx) 
    &= \prod_{l=1}^{L} \frac{\exp(\beta \cdot q_l(\bx,a_l))}{\sum_{a \in \calA_l} \exp(\beta \cdot q_l(\bx,a))},  \\
    \pi(\bs|\bx)  
    &= \prod_{l=1}^{L} \left( (1-\epsilon) \ind \{a_l = \underset{a \in \calA_l} {\operatorname{argmax}} \ q_l(\bx,a)\} + \frac{\epsilon}{|\calA_l|} \right),
\end{align}
where we use $\epsilon = 0.2$. We set $\beta=-0.05$ in Figure \ref{fig:sample_size_slate} and $\beta=-0.1$ in Figure \ref{fig:slate_size}.
\\\\
\textbf{Results.} \quad In Figure \ref{fig:sample_size_slate} and \ref{fig:slate_size}, we show comparisons of the estimators' performance on synthetic slate experiments in order to demonstrate that OPCB works better than the baselines in slate settings. We compare OPCB with DM, IPS, DR, PI, and LIPS, where PI and LIPS are defined rigorously in Appendix \ref{app:slate_ope}.  
We randomly select a slate abstraction function for LIPS. We vary logged data sizes ($n$) and slate sizes ($L$). We set $n = 1000$ and $|\calA_l|=3$ as default. Note that the shaded regions in the MSE plots represent the 95\% confidence intervals estimated with bootstrap.
\\\\
\textbf{How does OPCB perform with varying logged data sizes and slate sizes?} \quad 
In Figure \ref{fig:sample_size_slate}, we vary logged data sizes ($n$) from 500 to 8000. We observe that OPCB improves the MSE compared to the existing methods. Specifically, OPCB reduces variance against IPS and DR, and it has a lower bias than DM. On the other hand, PI achieves a significant variance reduction, yet it suffers from substantial bias because of violating the linearity assumption, which requires that the reward function is linearly decomposable. Similarly to PI, LIPS struggles with high bias since its slate abstraction does not retain sufficient information to characterize the expected reward function. Against PI and LIPS, OPCB provides a substantial reduction in bias and variance as it does not introduce some strict assumptions. These results support the advantage of OPCB in the slate settings. 
\\\\
Figure \ref{fig:slate_size} shows how slate sizes affect the performance of the estimators with varying slate sizes ($L$) from 4 to 7. The results show that all estimators become worse in a larger slate size as expected, and OPCB provides substantial improvements across various slate sizes. Particularly, IPS and DR degrade due to severe variance, and DM, PI, and LIPS suffer from high bias. OPCB achieves effective estimations by reducing bias and variance, resulting from using the marginalized importance weight regarding the main actions and the two-stage regression. The results suggest that OPCB can be more useful in practical and challenging situations.
\\\\

\begin{figure}[ht]
    \includegraphics[scale=0.25]{fig/legend.png} \vspace{1mm}
    \includegraphics[width=0.9\linewidth]{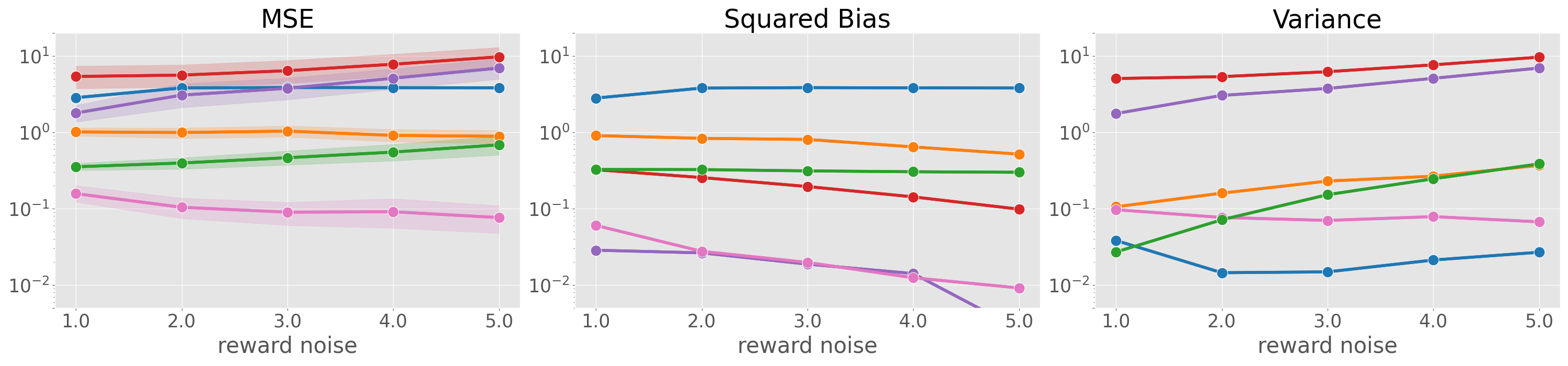}
    \centering \vspace{-3mm}
    \Description{}
    \caption{Comparisons of the estimator's MSE, squared bias, and variance with varying reward noise levels ($\sigma$).}
    \centering
    \label{fig:reward_std}
    \vspace{10mm}
\end{figure}

\begin{figure}[ht]
    \includegraphics[scale=0.25]{fig/legend.png} \vspace{1mm}
    \includegraphics[width=0.9 \linewidth]{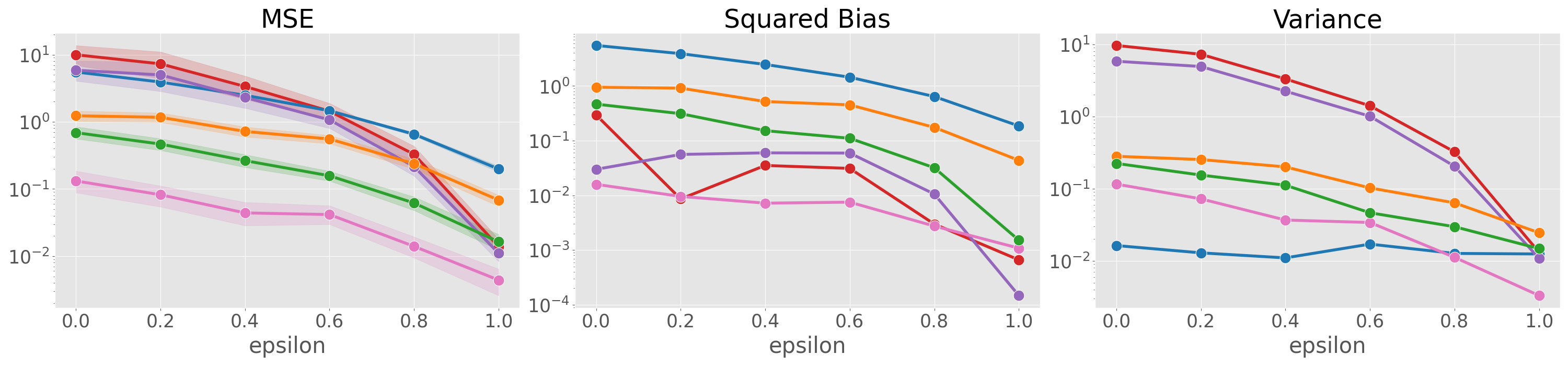}
    \centering
    \Description{}
    \caption{Comparisons of the estimator's MSE, squared bias, and variance with varying target policies ($\epsilon$).}
    \centering
    \label{fig:target_policy}
    \vspace{10mm}
\end{figure}

\begin{figure}[ht]
    \includegraphics[scale=0.25]{fig/legend.png} \vspace{1mm}
    \includegraphics[width=0.9 \linewidth]{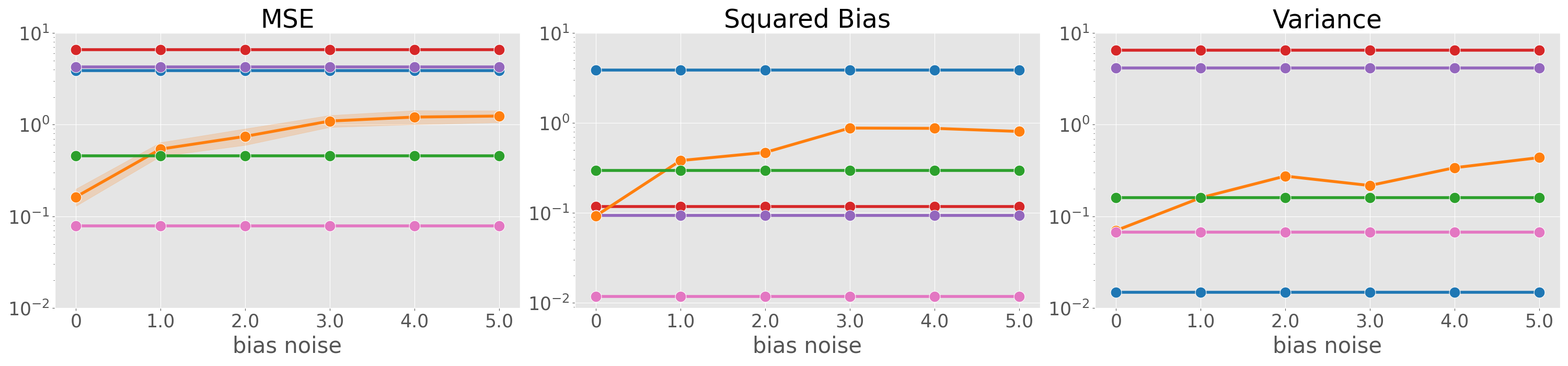}
    \centering
    \Description{}
    \caption{Comparisons of the estimator's MSE, squared bias, and variance with varying noise levels of bias estimation ($\sigma$).}
    \centering
    \label{fig:bias_noise}
    \vspace{10mm}
\end{figure}

\begin{figure}[ht]
    \includegraphics[scale=0.25]{fig/legend_real.png} \vspace{1mm}
    \includegraphics[width=0.9 \linewidth]{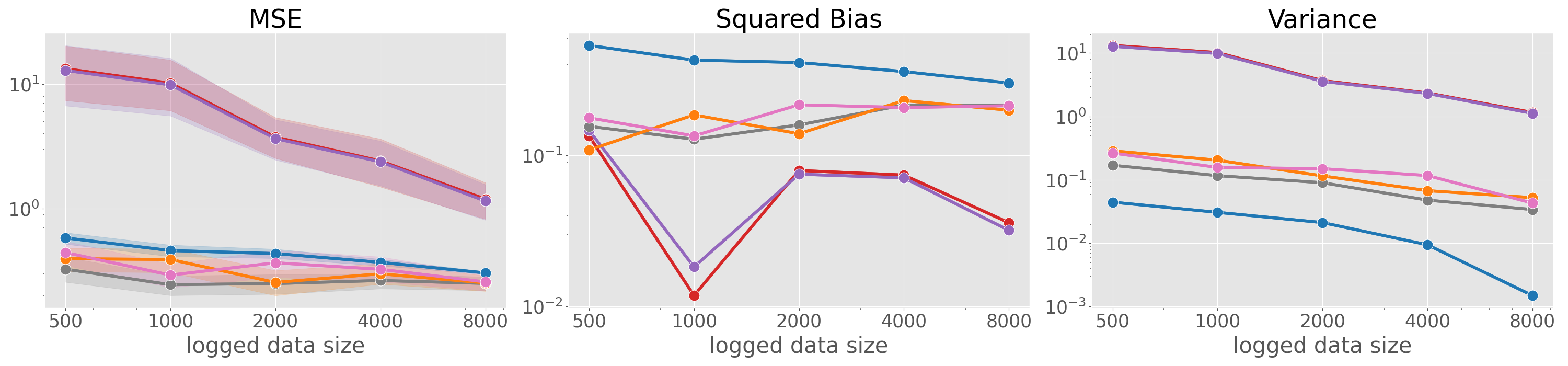}
    \centering
    \Description{}
    \caption{Comparisons of the estimator's MSE, squared bias, and variance with varying logged data sizes ($n$) on PenDigits.}
    \centering
    \label{fig:pen}
    \vspace{10mm}
\end{figure}

\begin{figure}[ht]
    \includegraphics[scale=0.25]{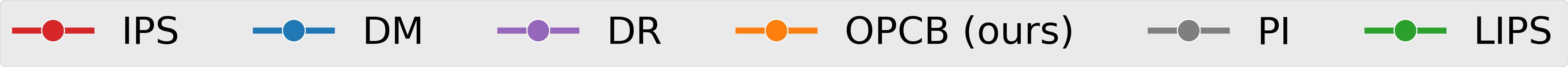} \vspace{1mm}
    \includegraphics[width=0.9 \linewidth]{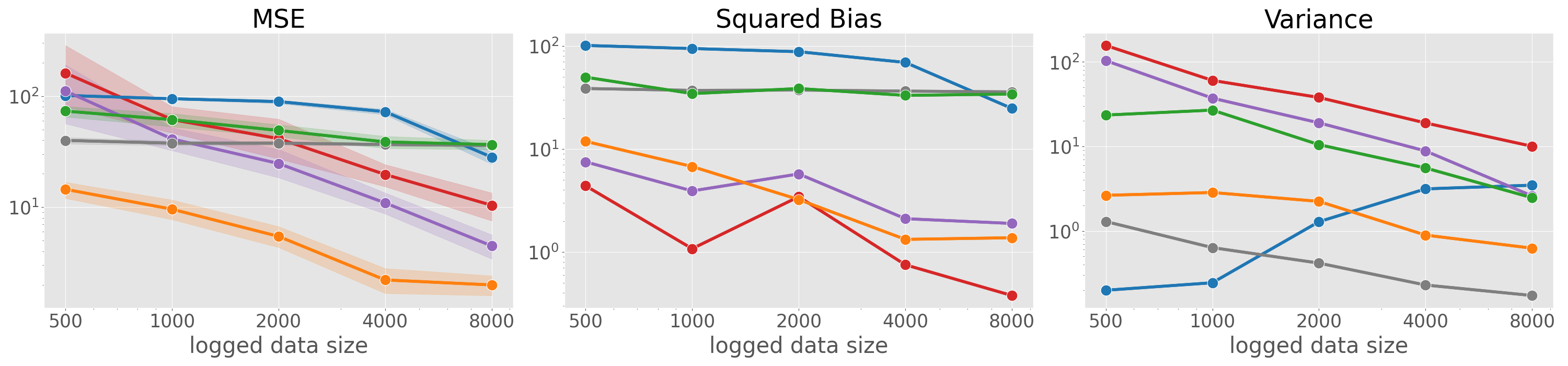}
    \centering
    \Description{}
    \caption{Comparisons of the estimator's MSE, squared bias, and variance with varying logged data sizes ($n$).}
    \centering
    \label{fig:sample_size_slate}
    \vspace{10mm}
\end{figure}

\begin{figure}[ht]
    \includegraphics[scale=0.25]{fig/legend_slate.png} \vspace{1mm}
    \includegraphics[width=0.9 \linewidth]{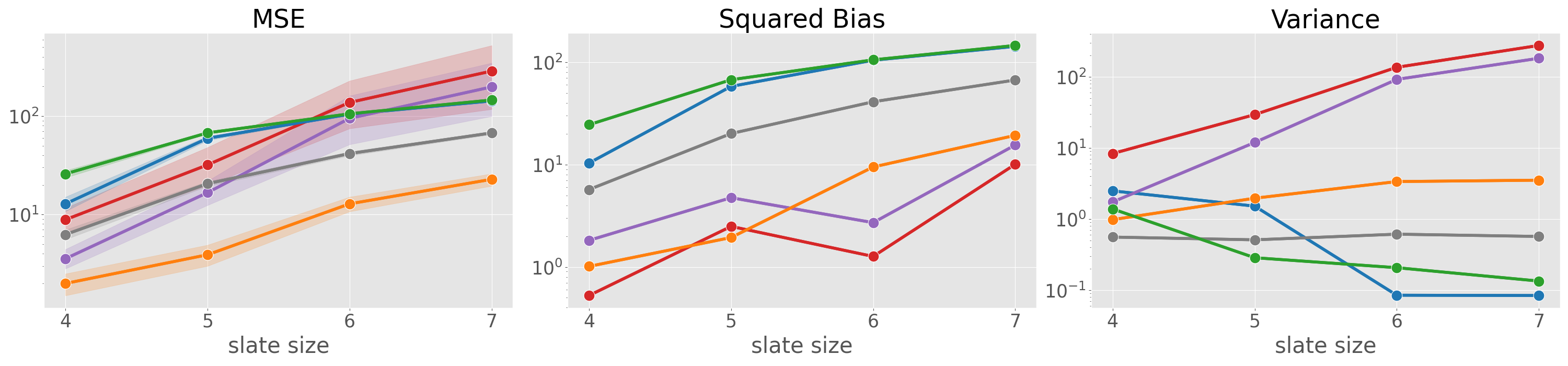}
    \centering
    \Description{}
    \caption{Comparisons of the estimator's MSE, squared bias, and variance with varying slate sizes ($L$).}
    \centering
    \label{fig:slate_size}
    \vspace{10mm}
\end{figure}

\end{document}